\documentclass{alt2019}

\title[Uniform regret bounds over $\mathbb{R}^d$ for linear regression]{Uniform regret bounds over $\mathbb{R}^d$ \\
for the sequential linear regression problem with the square loss}

\author{\Name{Pierre Gaillard} \Email{pierre.gaillard@inria.fr} \\
\addr INRIA, ENS, PSL Research University Paris, France\\
\AND
\Name{S{\'e}bastien Gerchinovitz} \Email{sebastien.gerchinovitz@math.univ-toulouse.fr} \\
\addr Institut de math{\'e}matiques de Toulouse, Universit{\'e} Paul Sabatier, Toulouse \\
\AND
\Name{Malo Huard} \Email{malo.huard@math.u-psud.fr} \\
\Name{Gilles Stoltz} \Email{gilles.stoltz@math.u-psud.fr} \\
\addr Laboratoire de math{\'e}matiques d'Orsay, Universit{\'e} Paris-Sud,
CNRS, Universit{\'e} Paris-Saclay, 91\,405 Orsay, France}

\usepackage{times}

\usepackage{scalerel}
\usepackage{bigints}
\usepackage{dsfont}
\usepackage{enumitem,multicol}
\usepackage{cancel}

\newtheorem{notation}[theorem]{Notation}
\newenvironment{rema}{\begin{remark} \em}{\end{remark}}
\newenvironment{rematitre}[1]{\begin{remark}[{#1}] \em}{\end{remark}}

\renewcommand{\log}{\ln}

\renewcommand{\leq}{\leqslant}
\renewcommand{\geq}{\geqslant}

\renewcommand{\phi}{\varphi}
\renewcommand{\epsilon}{\varepsilon}
\renewcommand{\hat}{\widehat}

\newcommand{\cF}{\mathcal{F}}

\newcommand{\cO}{\mathcal{O}}

\newcommand{\E}{\mathbb{E}}
\newcommand{\N}{\mathbf{N}}

\newcommand{\R}{\mathbb{R}}

\newcommand{\e}{\mathrm{e}}

\newcommand{\indicator}[1]{\mathds{1}_{#1}}
\newcommand{\ind}[1]{\indicator{#1}}
\renewcommand{\d}{\,\mathrm{d}}

\renewcommand{\tilde}{\widetilde}
\newcommand{\argmin}{\mathop{\mathrm{argmin}}}

\newcommand{\norm}[1]{\left\Vert #1 \right\Vert}
\newcommand{\bnorm}[1]{\bigl\Vert #1 \bigr\Vert}

\newcommand{\transp}{\mbox{\tiny \textup{T}}}

\newcommand{\eqdef}{\stackrel{\text{\rm def}}{=}}

\newcommand{\ts}{\textsuperscript}
\makeatletter
\newcommand{\raisemath}[1]{\mathpalette{\raisem@th{#1}}}
\newcommand{\raisem@th}[3]{\raisebox{#1}{$#2#3$}}
\makeatother

\newcommand{\y}{y}

\newcommand{\Y}{\mathbf{y}}
\newcommand{\yhat}{\hat{\y}}
\newcommand{\x}{\mathbf{x}}
\newcommand{\bz}{\mathbf{z}}
\newcommand{\xt}{\tilde{\mathbf{x}}}
\newcommand{\X}{\mathbf{X}}
\newcommand{\G}{\mathbf{G}}
\newcommand{\B}{\mathbf{b}}
\newcommand{\A}{\mathbf{A}}
\newcommand{\M}{\mathbf{M}}
\newcommand{\cI}{\mathcal{I}}
\newcommand{\I}{\mathbf{I}}
\newcommand{\Id}{\mathbf{I}_d}
\newcommand{\Sig}{\mathbf{\Sigma}}
\newcommand{\U}{\mathbf{U}}
\newcommand{\V}{\mathbf{V}}
\newcommand{\by}{\mathbf{y}}
\newcommand{\ud}{\mathbf{u}}
\newcommand{\vd}{\mathbf{v}}
\newcommand{\uhat}{\hat{\mathbf{u}}}
\newcommand{\hatu}{\hat{u}}
\newcommand{\ut}{\tilde{\mathbf{u}}}
\newcommand{\uc}{\breve{\mathbf{u}}}
\newcommand{\Rd}{\R^d}
\newcommand{\rT}{r_{\scaleto{T}{3pt}}}
\newcommand{\RrT}{\R^{\raisemath{2pt}{\rT}}}
\newcommand{\Gam}{\mathbf{\Gamma}}

\newcommand{\image}[1]{\mathop{\mathrm{Im}}({#1})}
\DeclareMathOperator{\Tr}{Tr}
\newcommand{\regret}{\mathcal{R}}
\newcommand{\eigen}{\lambda}

\newcommand{\onorm}[1]{{\left\vert\kern-0.25ex\left\vert\kern-0.25ex\left\vert #1 \right\vert\kern-0.25ex\right\vert\kern-0.25ex\right\vert}}
\newcommand{\jump}{\mathcal{T}}
\newcommand{\pot}{L^{\mbox{\tiny reg}}}
\newcommand{\transpose}[1]{{#1}^{\transp}}
\renewcommand{\top}{\text{\tiny \textup{T}}}
\DeclareMathOperator*{\rank}{rank}

\DeclareMathOperator{\spn}{span}

\newcommand{\ev}{\mathbf{e}}
\newcommand{\thetav}{\mathbf{\theta}}

\begin{document}

\maketitle

\begin{abstract}
We consider the setting of online linear regression for arbitrary deterministic sequences,
with the square loss. We are interested in the aim set by \citet{Bartlett2015}:
obtain regret bounds that hold uniformly over all competitor vectors.
When the feature sequence is known at the beginning of the game, they provided closed-form regret bounds of $2d B^2 \ln T + \cO_T(1)$, where $T$ is the number of rounds and $B$ is a bound on the observations. Instead, we derive bounds with an optimal constant of $1$ in front of the $d B^2 \ln T$ term. In the case of sequentially revealed features, we also derive an asymptotic regret bound of $d B^2 \ln T$ for any individual sequence of features and bounded observations.
All our algorithms are variants of the online non-linear ridge regression forecaster, either with a data-dependent regularization or with almost no regularization.
\end{abstract}

\begin{keywords}
Adversarial learning, regret bounds, linear regression, (non-linear) ridge regression
\end{keywords}

    \section{Introduction and setting}

We consider the setting of online linear regression for arbitrary deterministic sequences with the square loss, which unfolds as follows. First, the environment chooses a sequence of observations $(y_t)_{t\geq 1}$ in $\R$ and a sequence of feature vectors $(\x_t)_{t\geq 1}$ in $\R^d$. The observation sequence $(y_t)_{t\geq 1}$ is initially hidden to the learner, while the sequence of feature vectors (see \citealp{Bartlett2015}) may be given in advance or be initially hidden as well, depending on the setting considered: ``beforehand-known features'' (also called the fixed-design setting) or ``sequentially revealed features''.
At each forecasting instance $t\geq 1$, Nature reveals $\x_t$ (if it was not initially given), then the learner forms a prediction $\hat y_t \in \R$. The observation $y_t \in \R$ is then revealed and instance $t+1$ starts. In all results of this paper, the observations $y_t$ will be assumed to
be bounded in $[-B,B]$ (but the forecaster will have no knowledge of $B$),
while we will avoid as much as possible boundedness assumptions of the features $\x_t$.
See Figure~\ref{fig:setting}.

\begin{figure}[!t]
	\begin{center}
		\begin{tabular}{lcl}
			\hline
			{Sequentially revealed features} & \phantom{space} & {Beforehand-known features} \\
			\hline \\[-12pt]
			\textbf{Given:} [No input] & & \textbf{Given:} $\x_1,\ldots, \x_T \in \Rd$ \\
			\textbf{For} $t = 1, 2, \ldots, T$, the learner: & &\textbf{For} $t = 1, 2, \ldots, T$, the learner:\\
			\quad \textbullet~ observes $\x_t \in \Rd$ & & \\
			\quad \textbullet~ predicts $\yhat_t \in \R$ & & \quad \textbullet~ predicts $\yhat_t \in \R$ \\
			\quad \textbullet~ observes $\y_t \in [-B,B]$ & & \quad \textbullet~ observes $\y_t \in [-B,B]$  \\
			\quad \textbullet~ incurs $(\yhat_t - \y_t)^2 \in \R$ & & \quad \textbullet~ incurs $(\yhat_t - \y_t)^2 \in \R$ \\
			\hline
		\end{tabular} \vspace{-.4cm}
	\end{center}
  \caption{The two online linear regression settings considered, introduced by \citet{Bartlett2015}; the learner
  has no knowledge neither of $B$ nor (in the left case) of $T$.}
  \label{fig:setting}
\end{figure}

The goal of the learner is to perform on the long run (when $T$ is large enough) almost as well as the best fixed linear predictor in hindsight. To do so, the learner minimizes her cumulative regret,
\[
\regret_T(\ud) = \sum_{t=1}^T (\y_t - \yhat_t)^2 - \sum_{t=1}^T (\y_t - \ud \cdot \x_t)^2\,,
\]
either with respect to specific vectors $\ud \in \R^d$ (e.g., in a compact subset) or uniformly over $\R^d$.
In this article, and following \citet{Bartlett2015}, we will be interested in
\begin{equation}
\label{eq:defunifregret}
\sup_{\ud \in \R^d} \regret_T(\ud) = \sum_{t=1}^T (\y_t - \yhat_t)^2 - \inf_{\ud \in \R^d} \sum_{t=1}^T (\y_t - \ud \cdot \x_t)^2\,,
\end{equation}
which we will refer to as the uniform regret over $\R^d$ (or simply, the uniform regret).
The worst-case uniform regret corresponds to the largest uniform regret
of a strategy, when considering all possible sequences of features~$\x_t$ and (bounded) observations~$y_t$;
we will also refer to it as a twice uniform regret, see Section~\ref{sec:openquestion}.

\paragraph{Notation.} Bounded sequences of real numbers $(a_t)_{t \geq 1}$, possibly depending on external quantities like the feature
vectors $\x_1,\x_2,\ldots$, are denoted by $a_T = \cO_T(1)$. For a given positive function $f$, the piece of notation
$a_T = \cO_T\bigl( f(T) \bigr)$ then indicates that $a_T/f(T) = \cO_T(1)$. Also, the notation $a_T = \Theta_T(1)$
is a short-hand for the facts that $a_T = \cO_T(1)$ and $1/a_T = \cO_T(1)$, i.e., for the fact
that $(a_t)_{t \geq 1}$ is bounded from above and from below. We define a similar extension
$a_T = \Theta_T\bigl( f(T) \bigr)$ meaning that $a_T/f(T) = \Theta_T(1)$.

\paragraph{Earlier works.}
Linear regression with batch stochastic data has been extensively studied by the statistics community. Our setting of online linear regression for arbitrary sequences is of more recent interest; it dates back to~\citet{Foster1991}, who considered binary labels $y_t \in \{0,1\}$ and vectors $\ud$ with bounded $\ell_1$--norm. We refer the interested reader to the monograph by~\citet[Chapter~11]{Cesa-Bianchi2006} for a thorough introduction to this literature and to
\citet{Bartlett2015} for an overview of the state of the art. Here, we will mostly highlight some key contributions. One is by \citet{Vovk01} and \citet{AzouryWarmuth2001}: they designed the non-linear ridge regression recalled in Section~\ref{sec:NLVovk},
which achieves a regret of order $d \log T$
only uniformly over vectors $\ud$ with bounded $\ell_2$--norm. \citet{Vovk01} also provided a matching minimax lower bound  $dB^2 \log T - \cO_T(1)$ on the worst-case uniform regret over $\R^d$ of any forecaster, where $B$ is a bound on the observations $|y_t|$ (see also the lower bound provided by~\citealp{TW00}). More recently,
\citet{Bartlett2015} computed the minimax regret for the problem with beforehand-known features and provided an algorithm that is optimal under some (stringent) conditions on the sequences $(\x_t)_{t\geq 1}$ and $(y_t)_{t \geq 1}$
of features and observations. The best closed-form (but general: for all sequences) uniform regret they could obtain
for this algorithm was of order $2 d B^2 \log T$.
This algorithm is scale invariant with respect to the sequence of features $(\x_t)_{t\geq 1}$. Their analysis emphasizes the importance of a data-dependent metric to regularize the algorithm, which is harder to construct when the features are only revealed sequentially. To that end, \citet{MalekHorizonIndependentMinimaxLinear2018} show that, under quite intricate constraints on the features and observations, the backward algorithm of \citet{Bartlett2015} can also be computed in a forward (and thus legitimate) fashion in the case when the features are only revealed sequentially. It is thus also optimal; see, e.g., Lemma~39 and Theorem~46 therein.

\paragraph{Organization of the paper and contributions.}
We first recall and discuss the regret bound of the non-linear ridge regression algorithm
(Section~\ref{sec:NLVovk}), whose proof will be a building block for our new analyses;
we will show that perhaps surprisingly it enjoys a uniform regret bound $2dB^2 \log T + \cO_T(1)$.
For the sake of completeness, we also state and re-prove (Section~\ref{sec:NLVovk-COLS}) the regret lower bound by~\citet{Vovk01} (and \citealp{TW00}),
as most of the discussions in this paper will be about the optimal constant in
front of the $d B^2 \ln T$ regret bound; this optimal constant will be seen to equal $1$.
Our proof resorts to a general argument, namely, the van Trees inequality (see~\citealp{GiLe95}), to lower bound the error made by any forecaster, while \citet{Vovk01} was heavily relying on the fact that in a Bayesian stochastic context,
the optimal strategy can be determined. This new tool for the machine learning community could be of
general interest to derive lower bounds in other settings.
We also believe that our lower bound proof is enlightening for statisticians. It shows that the expectation of the regret is larger than a sum of quadratic estimation errors for a $d$--dimensional parameter. Each of these errors corresponds to an estimation based on a sample of respective length~$t-1$, thus is larger than something of the order of $d/t$, which is the optimal parametric estimation rate.
Hence the final $d (1 + 1/2 + \ldots + 1/T) \sim d \ln T$ regret lower bound.

We next show (Section~\ref{sec:beforehandfeatures})
that in the case of beforehand-known features, the non-linear ridge regression algorithm
and its analysis may make good use of a proper metric $\norm{\,\cdot\,}_{\G_T}$ described
in~\eqref{eq:nlridgeadapted-bis} instead of the Euclidean norm. This leads to a worst-case bound of
$d B^2 \ln(1+T/d) + d B^2$ on the uniform regret over $\R^d$, which is optimal (with an optimal constant of $1$) in view of the perfectly matching minimax lower bound of Section~\ref{sec:NLVovk-COLS}. To the best
of our knowledge, earlier closed-form worst-case upper bounds were suboptimal by a factor of $2$. See the corresponding discussions
for the non-linear ridge regression algorithm, in Section~\ref{sec:NLVovk},
and for the minimax forecaster by~\citet{Bartlett2015}, in Remark~\ref{rk:Bartlett} of Section~\ref{sec:beforehandfeatures}.

The question then is (Section~\ref{sec:sequentialfeatures})
whether a $dB^2 \log T + \mathcal{O}_T(1)$ regret bound can be achieved on the uniform regret
in the most interesting setting of sequentially revealed features.
Surprisingly enough, even if the traditional bound for the
non-linear ridge regression forecaster blows up
when the regularization parameter vanishes, $\lambda = 0$
(see Section~\ref{sec:NLVovk}), an ad hoc analysis can be made in this case;
it yields a uniform regret bound of $dB^2\log T + \mathcal{O}_T(1)$.
This bound holds for any fixed sequence of features and bounded observations; we do not impose stringent conditions as in \citet{MalekHorizonIndependentMinimaxLinear2018}.
Also, no parameter needs to be tuned, which is a true relief.
The only drawback of this bound,
compared to the bounds obtained in the case of beforehand-known features, is that the $\mathcal{O}_T(1)$ remainder
term depends on the sequence of features. We thus could not derive a worst-case uniform regret bound.

Therefore, a final open question is stated in Section~\ref{sec:openquestion} and
consists in determining if such a twice uniform regret bound of order $d B^2 \ln T$ (over comparison vectors $\ud \in \R^d$
and over feature vectors $\x_t \in \R^d$ and bounded observations $y_t$) may hold in the case
of sequentially revealed features, or whether the lower bound should be improved.
The proofs of the lower bound (the ones by \citealp{Vovk01}, \citealp{TW00}, and our one)
generate observations and feature vectors ex ante, independently of the strategy considered,
and reveal them to the latter before the prediction game starts.
However, it might be the case that truly sequential choices to annoy the strategy considered
or generating feature sequences with difficult-to-predict sequences of Gram matrices
lead to a larger regret being suffered.

	\section{Sequentially revealed features / Partially known results}
  \label{sec:recall}
    In this section, we recall and reestablish some known results regarding the regret with the square loss function. We recall the definition and the regret bound (Section~\ref{sec:NLVovk}) of the non-linear ridge regression algorithm of~\citet{Vovk01}, \citet{AzouryWarmuth2001}. The (proof of this) regret bound is used later in this article to design and study our new strategies. We reestablish as well a
    \[
    d B^2 \bigl( \ln T - (3+\ln d) - \ln \ln T \bigr)
    \]
    lower bound on the regret of any forecaster (Section~\ref{sec:NLVovk-COLS}),
    which implies that the worst-case uniform regret bound $d B^2 \ln(1+T/d) + d B^2$ obtained
		in Section~\ref{sec:beforehandfeatures} is first-order optimal: it gets the optimal $d B^2 \ln T$ main term.

    \subsection{Upper bound on the regret / Reminder of a known result + a new consequence of it}
    \label{sec:NLVovk}

	The \emph{non-linear ridge regression algorithm} of Vovk, Azoury and Warmuth uses at each time-step $t$ a vector $\uhat_t$ such that $\uhat_1 = (0,\ldots,0)^{\transp}$ and for $t \geq 2$,
	\begin{equation}
	\label{eq:nlridge}
	\uhat_t \in \argmin_{\ud \in \Rd} \left\{ \sum_{s=1}^{t-1} (\y_s - \ud \cdot \x_s)^2
	+ (\ud \cdot \x_t)^2 + \lambda \norm{\ud}^2 \right\},
	\end{equation}
	where $\norm{\,\cdot\,}$ denotes the Euclidean norm, and predicts $\yhat_t = \uhat_t \cdot \x_t$.
No clipping can take place to form the prediction as the learner has no knowledge of the range $[-B,B]$
of the observations.

	Note that the definition~\eqref{eq:nlridge} is not scale invariant. By scale invariance, we mean that if the $\x_t$ are all
multiplied by some $\gamma > 0$ (or even by an invertible matrix~$\Gamma$), the vector $\uhat_t$ used should also be just divided by $\gamma$ (or multiplied
by~$\Gamma^{-1}$). We may also define what a scale-invariant bound on the uniform regret is: a bound that is unaffected
by a rescaling of the feature vectors $\x_t$ (as the vectors $\ud \in \R^d$ compensate for the rescaling).

	\begin{notation}\label{not1}
		Given features $\x_1,\x_2,\ldots \in \R^d$, we denote by
		$\G_t = \sum_{s=1}^t \x_s \transpose{\x}_s$ the associated $d \times d$ Gram matrix at step $t \geq 1$.
		This matrix is symmetric and positive semidefinite; it admits $d$ eigenvalues, which we sort in non-increasing order
		and refer to as $\eigen_1(\G_t), \, \ldots, \, \eigen_d(\G_t)$. Furthermore, we denote by
        $r_t = \rank(\G_t)$ the rank of $\G_t$. In particular, $\eigen_{r_t}(\G_t)$ is the smallest
        positive eigenvalue of $G_t$.
	\end{notation}

	For $\lambda > 0$, we have a unique, closed-form solution of~\eqref{eq:nlridge}:
    denoting $\A_{t} = \lambda \, \Id + \G_{t}$, which is a symmetric definite positive thus invertible matrix,
	and $\B_{t-1} = \sum_{s=1}^{t-1} \y_t \, \x_t$,
	\begin{equation}
	\label{eq:Ab1}
	\uhat_t = \A_{t}^{-1} \B_{t-1}\,.
	\end{equation}

	We recall the proof of the following theorem in Appendix~\ref{apd:thm1proof}, mostly for the sake of completeness
and because we will use some standard inequalities extracted from it.

	\begin{theorem}[see Theorem~\textbf{11.8} of \citealp{Cesa-Bianchi2006}]
		\label{thm:nlridge}
		Let the non-linear ridge regression~\eqref{eq:nlridge} be run with parameter $\lambda > 0$.
		For all $T \geq 1$, for all sequences $\x_1,\ldots,\x_T \in \R^d$ and all $\y_1,\ldots,y_T \in [-B,B]$,
		for all $\ud \in \Rd$,
		\[
        \regret_T(\ud) \leq \lambda \norm{\ud}^2 + B^2 \sum_{k=1}^d \log \!\left(1 + \frac{\eigen_{k}(\G_T)}{\lambda} \right).
        \]
	\end{theorem}

The regret bound above involves a $\lambda \norm{\ud}^2$ term, which blows up when the
supremum over $\ud \in \R^d$ is taken. However, under an additional boundedness assumption
on the features $\x_t$, we could prove the following
uniform regret bound. To the best of our knowledge, this is the first uniform regret bound
proved for this well-known forecaster. Other uniform regret bounds (see \citealp{Bartlett2015})
were proved for ad-hoc and more involved forecasters, not for a standard, good old forecaster
like the non-linear ridge regression~\eqref{eq:nlridge}.

However, despite our best efforts, the uniform regret bound
we could prove is only of the form $2dB^2\log(T)+\cO_T(1)$. It has two drawbacks:
first, as we show in the next sections, the constant $2$ in the leading term
is suboptimal; second, the $\cO_T(1)$ strongly depends on the sequence of feature vectors.
The proof is provided in Section~\ref{apd:thm1proof}
and essentially consists in noting that it is unnecessary to
worry about vectors $\ud \in \R^d$ with too large a norm, as they
never achieve the infimum in~\eqref{eq:defunifregret}.

\begin{corollary}
\label{cor:RNL-standard}
Let the non-linear ridge regression~\eqref{eq:nlridge} be run with parameter $\lambda > 0$.
For all $T \geq 1$, for all sequences $\x_1,\ldots,\x_T \in \R^d$ with $\norm{\x_t} \leq X$
and all $\y_1,\ldots,y_T \in [-B,B]$,
\[
\sup_{\ud \in \R^d} \regret_T(\ud) \leq r_T B^2 \log \!\left(1 + \frac{T X^2}{r_T \lambda} \right)
+ \frac{\lambda}{\lambda_{r_T}(\G_T)} T B^2 \,.
\]
\end{corollary}

Proper choices for $\lambda$ to minimize the upper bound above are roughly of the order of $1/T$,
to get rid of the linear part of the bound given by $T B^2$;
because of the $T/\lambda$ term in the logarithm, the resulting bound has unfortunately a main term of order
$d B^2 \ln T^2 = 2 d B^2 \ln T$. For instance, the choice $\lambda = 1/T$, that does
not require any beforehand knowledge of the features $\x_t$, together with the bound $r_T \leq d$
and the fact that $u \mapsto (1/u)\ln(1+u)$ is decreasing over $(0,+\infty)$, leads to
a regret bound less than
\[
\qquad 2 d B^2 \ln T + \frac{B^2}{\lambda_{r_T}(\G_T)} + d B^2 \log \bigl( 1 + X^2/d \bigr)\,.
\]
The $B^2/\lambda_{r_T}(\G_T)$ quantity in the regret bound is not uniformly bounded
over sequences of features~$\x_t$. In this respect, Corollary~\ref{cor:RNL-standard}
only constitues a minor improvement on Theorem~\ref{thm:nlridge}. We note a scaling
issue: for a fixed sequence of observations $y_1,y_2,\ldots$, while the uniform regret
is not affected by a scaling of the feature vectors, the upper bound exhibited above
is so. The deep reason for this issue is the lack of invariance of the non-linear
ridge regression~\eqref{eq:nlridge} itself.

    \subsection{Lower bound on the uniform regret / Improvement on known results}
    \label{sec:NLVovk-COLS}

    In this section, we study the uniform regret
    $\displaystyle{\sup \bigl\{ \regret_T(\ud) : \ud \in \R^d \bigr\}}$, in
    the minimax case (of befo\-re\-hand-known features, which is the most difficult setting
    for a lower bound). That is, we are interested in
    \begin{equation}
    \label{eq:defregretstar}
    \regret^\star_{T,\,[-B,B]} \eqdef
    \inf_{\text{forecasters}}
    \ \sup_{\x_1,\ldots,\x_T \in \times [0,1]^d}
    \ \sup_{y_t \in [-B,B]} \ \left\{
    \sum_{t=1}^T (y_t - \hat y_t)^2 - \inf_{\ud \in \R^d} \sum_{t=1}^T (y_t - \ud \cdot \x_t)^2 \right\} ,
    \end{equation}
    where the first infimum is over all forecasters (all forecasting strategies) that can possibly access beforehand
    all the features $\x_1,\ldots,\x_T$ that are considered next, and the second supremum is over all
    individual sequences $y_1,\ldots,y_T \in [-B,B]$, that are sequentially revealed.
    Our result is the following; we carefully explain in Remark~\ref{rk:LBothers} why this result
    slightly improves on the existing literature.

    \begin{theorem}
    \label{thm:lowerbound}
    For all $T \geq 8$ and $B > 0$, we have $\regret^\star_{T,\,[-B,B]}
    \geq d B^2 \bigl( \ln T - (3+\ln d) - \ln \ln T \bigr)$.
    \end{theorem}

    \begin{rema}
    Note that the features $\x_t$ could be any element of $\R^d$ (by a scaling property on the $\ud$),
    they do not necessarily need to
    be restricted to $[0,1]^d$; it is merely that our proof relies on such $[0,1]^d$--valued
    features. Compare to Theorem~\ref{thm:nlridgeadapted}, where no boundedness assumption is required
    on the features.
    \end{rema}

\begin{rema}
\label{rk:LBothers}
Our proof reuses several ideas from the original proof of \citet[Theorem~2]{Vovk01}, namely,
taking features $\x_t$ with only one non-zero input equal to 1 and Bernoulli observations $y_t$,
resorting to a randomization with a Beta prior distribution, etc.; see also the proof
of \citet[Theorem~4]{TW00}. However, we believe that we achieve a more satisfactory result than
the $(d-\varepsilon) B^2 \ln T - C_\varepsilon$ lower bound of \citet[Theorem~2]{Vovk01},
where $\varepsilon > 0$ is a parameter and $C_\varepsilon$ is a finite value; also,
the proof technique somewhat relied on the boundedness of the features to derive the general case $d \geq 2$
from the special case $d=1$.
See also similar results for the case $d=1$ in \citet[Theorem~4]{TW00},
with the same issue for the generalization to $d \geq 2$.
Our proof, on the contrary, directly tackles the $d$--dimensional case, which turns out to be
more efficient (and more elegant). However, our alternative proof for the lower bound
is admittedly a minor variation of existing results, it merely sheds a slightly
different light on the bound, see the interpretation below in terms
of parametric estimation rate.
\end{rema}

    The high-level idea of our proof of this known bound is to see the desired $d \ln T$ bound
    as a sum of parametric estimation errors in $\R^d$, each of order at least $d/t$.
    It is a classic result in parametric statistics that the estimation of a $d$--dimensional
    parameter based on a sample of size $t$ can be performed at best at rate $d/t$
    in quadratic error, and this is exactly what is used in our proof.
\citet[Theorem~2]{Vovk01} was heavily relying on the fact that in a Bayesian stochastic context, the optimal strategy can be determined: his proof states that
``since Nature's strategy is known, it is easy to find the best, on the average,
strategy for Statistician (the Bayesian strategy).''
In contrast, our argument does not require to explicitly compute the optimal strategy.
It relies on the van Trees inequality (see~\citealp{GiLe95}),
that lower bounds the estimator error of any, possibly biased, forecaster---unlike the {C}ram{\'e}r-{R}ao bound,
which only holds for unbiased estimators.
In this respect, the van Trees inequality could reveal itself a new tool of general interest for the machine learning community
to derive lower bounds in other settings.  \medskip

\begin{proof} \!\!\!\! \textbf{(sketch)}
The complete proof can be bound in Appendix~\ref{sec:vT} and we merely indicate here its most salient
arguments. We start with a case where $y_t \in [0,1]$ and explain later
how to draw the result for the desired case where $y_t \in [-B,B]$.

We fix any forecaster. A sequence $J_1,\ldots,J_T$ is drawn independently and uniformly at
random over $\{1,\ldots,d\}$ and we associate with it the sequence of feature vectors $\ev_{J_1},\ldots,\ev_{J_T}$,
where $\ev_j$ denotes the unit vector $(0,\ldots,0,1,0,\ldots,0)^{\transp}$ along the $j$-th coordinate (the $1$ is in
position $j$). The forecaster is informed of this sequence of feature vectors and the sequential prediction problem starts.
We actually consider several prediction problems, each indexed by $\thetav^\star \in [0,1]^d$:
conditionally on the feature vectors $\ev_{J_1},\ldots,\ev_{J_T}$, at each round $t$ the observation $Y_t$ is drawn independently according to
a Bernoulli distribution with parameter $\thetav^\star \cdot \ev_{J_t} = \theta^\star_{J_t}$.
Expectations with respect to the randomization thus defined will be denoted by $\E_{\thetav^\star}$.

Now, given the features considered above, that are unit vectors, each forecasting strategy can be termed
as picking only linear combinations $\yhat_t = \uhat_t \cdot \ev_{J_t}$ as predictions. Indeed, we denote by
$\yhat_t(j)$ the prediction output by the strategy when $J_t = j$ given the past observations $Y_1,\ldots,Y_s$ and
the features $J_1,\ldots,J_T$. We then consider the vector $\uhat_t \in \R^d$ whose $j$--th component equals $\hatu_{j,t} =
\yhat_t(j)$. This way, in our specific stochastic setting, outputting direct predictions $\yhat_t$ of the observations
or outputting vectors $\uhat_t \in \R^d$ to form linear combinations are the same thing.

The sketchy part of the proof starts here (again, details can be found in Appendix~\ref{sec:vT}).
By exchanging an expectation and an infimum and by repeated uses of the tower rule, we have,
for each $\thetav^\star \in [0,1]^d$:
\begin{equation}
\label{eq:prendreesp}
\E_{\thetav^\star} \! \left[ \sup_{\ud \in \R^d} \regret_T(\ud) \right]
\geq \sum_{t=1}^T \E_{\thetav^\star} \Bigl[ (Y_t - \yhat_t)^2 \Bigr] -
\inf_{\ud \in \R^d} \sum_{t=1}^T \E_{\thetav^\star} \Bigl[ (Y_t - \ud \cdot \ev_{J_t})^2 \Bigr]
\geq \sum_{t=1}^T \frac{1}{d} \, \E_{\thetav^\star} \Bigl[ \bigl\Arrowvert \uhat_t - \theta^\star \bigr\Arrowvert^2_2 \Bigr]\,.
\end{equation}
The inequality above being valid for all forecasters accessing in advance to the entire sequence of feature vectors,
we thus proved, for any prior $\pi$ over $[0,1]^d$,
\begin{align*}
\regret^\star_{T,\,[0,1]} \geq \inf_{\text{forecasters}} \ \sum_{t=1}^T
\bigintsss_{[0,1]^d} \frac{1}{d} \, \E_{\theta^\star} \Bigl[ \bigl\Arrowvert \uhat_t - \theta^\star \bigr\Arrowvert^2_2 \Bigr] \d\pi(\theta^\star)\,.
\end{align*}
An immediate application of the (multi-dimentional) van Trees inequality with a
Beta($\alpha,\alpha$) prior $\pi$ shows that for all forecasters, all $t \geq 1$ and $\alpha \geq 3$,
\[
\bigintsss_{[0,1]^d} \frac{1}{d} \, \E_{\theta^\star} \Bigl[ \bigl\Arrowvert \uhat_t - \theta^\star \bigr\Arrowvert^2_2 \Bigr] \d\pi(\theta^\star) \geq \frac{d}{4t + 2t/(\alpha-1) + 16 d \alpha}\,,
\]
which is roughly of order $d/(4t)$, as we will take large values of $\alpha$ (of order $\ln T$).
Straightforward calculations conclude the proof and lead to a lower bound on $\smash{\regret^\star_{T,\,[0,1]}}$
of order $(d/4) \ln T$, which entails a lower bound of order $d B^2 \ln T$ on $\regret^\star_{T,\,[-B,B]}$.
\end{proof}

	\section{Beforehand-known features / New result}
	\label{sec:beforehandfeatures}

	In this section we assume that the features are known beforehand and exhibit a simple forecaster with a closed-form regret bound of $d B^2\log T+\cO_T(1)$ uniformly over $\Rd$ and all sequences of features and of bounded observations.
Combined with the minimax lower bound of Theorem~\ref{thm:lowerbound}, this upper bound implies that the minimax regret for beforehand-known features has a leading term exactly equal to $d B^2\log T$. It thus closes a gap between $d B^2\log T$ and $2 d B^2\log T$ left open by earlier closed-form results such as those of \citet[Theorem~8]{Bartlett2015}.
See a more detailed discussion below, in Remark~\ref{rk:Bartlett}.

	The \emph{non-linear ridge regression algorithm with adapted regularization} will pick
	weight vectors as follows:
	$\uhat_1 = (0,\ldots,0)^{\transp}$ and for $t \geq 2$,
	\begin{equation}
	\label{eq:nlridgeadapted}
	\uhat_t \in \argmin_{\ud \in \Rd} \left\{ \sum_{s=1}^{t-1} (\y_s - \ud \cdot \x_s)^2
	+ (\ud \cdot \x_t)^2 + \lambda \sum_{s=1}^T (\ud \cdot \x_s)^2 \right\}
	\end{equation}
	with the constraint that $\uhat_t$ should  be of minimal norm within all vectors of the
	stated $\argmin$. It then predicts $\yhat_t = \uhat_t \cdot \x_t$. As shown in Appendix~\ref{apd:thm2proof}, the closed-form expression
	for $\uhat_t$ reads
	\begin{equation}
	\label{eqn:nlridgeadaptedcf}
	\uhat_t = \bigl(\lambda\, \G_T + \G_t \bigr)^\dagger \B_{t-1}\,,
	\end{equation}
	where $\dagger$ denotes the Moore-Penrose inverse of a matrix (see Appendix~\ref{sec:MPI}) and where $\B_{t-1}$ was defined in~\eqref{eq:Ab1}. \medskip

	The difference to \eqref{eq:nlridge} lies in the regularization term, which can be denoted by
	\begin{equation}
    \label{eq:nlridgeadapted-bis}
	\lambda \norm{\ud}_{\G_T}^2 \eqdef \lambda \, \transpose{\ud} \G_T \ud
	= \lambda \sum_{s=1}^T (\ud \cdot \x_s)^2\,;
	\end{equation}
	that is, this regularization term can be seen as a metric adapted to the known-in-advance features $\x_1,\ldots,\x_T$. This algorithm has the desirable property of being scale invariant. Actually, as will be clear from the equality~\eqref{eq:whitening} in the proof of Theorem~\ref{thm:nlridgeadapted},
the strategy~\eqref{eq:nlridgeadapted} considered here consists of ``whitening'' the feature vectors
$\x_t$ into $\xt_t = \G_T^{-1/2} \x_t$ and applying the ``classic'' non-linear ridge regression~\eqref{eq:nlridge}
to these whitened feature vectors $\xt_t$. Their associated Gram matrix is the identity, which
helps obtaining a sharper bound from Theorem~\ref{thm:nlridge}
than the suboptimal but general bound obtained in Corollary~\ref{cor:RNL-standard}.

\begin{theorem}
\label{thm:nlridgeadapted}
Let the non-linear ridge regression algorithm with adapted regularization~\eqref{eq:nlridgeadapted} be run with parameter $\lambda > 0$.
For all $T \geq 1$, for all feature sequences $\x_1,\ldots,\x_T \in \R^d$ and all $\y_1,\ldots,y_T \in [-B,B]$,
\[
\qquad \sup_{\ud \in \R^d}
\regret_T(\ud) \leq \lambda T B^2 + r_T B^2 \log\Bigl(1 + \frac{1}{\lambda}\Bigr)\,,
\]
where $r_T =\rank(\G_T)$.
\end{theorem}

	By taking $\lambda = r_T/T$, we get the bound $r_T B^2 \bigl(1 + \log(1 + T/r_T) \bigr)$.
	Of course, $r_T \leq d$ and since $u \mapsto (1/u)\ln(1+u)$ is decreasing over $(0,+\infty)$,
    the final optimized regret bound reads
    \[
    \sup_{\ud \in \Rd} \regret_T(\ud) \leq B^2 \Biggl( r_T \log \biggl( 1 + \frac{T}{r_T} \biggr) + r_T \Biggr)
    \leq d \, B^2 \log \biggl( 1 + \frac{T}{d} \biggr) + d \, B^2\,.
    \]
	Note that the leading constant is $1$, which is known to be optimal because of Theorem~\ref{thm:lowerbound}.

\begin{rema}
\label{rk:Bartlett}
\citet{Bartlett2015} study some minimax uniform regret, namely
\[
\regret^\star_T
=
\sup_{\substack{\x_1,\dots,\x_T \in \R^d\\ \text{satisfying~\eqref{eq:cstr}}}} \inf_{\hat y_1} \sup_{y_1 \in [-B,B]} \cdots \, \inf_{\hat y_T} \sup_{y_T \in [-B,B]}
\,\, \sup_{\ud \in \Rd} \regret_T(\ud)\,,
\]
and design a forecaster called \texttt{MM} based on backward induction. It uses vectors $\uhat_t = \mathbf{P}_t \B_{t-1}$ for $t \geq 2$,
where the sequence $\mathbf{P}_1, \mathbf{P}_2, \ldots, \mathbf{P}_T$ is defined in a backward manner as
\[
\mathbf{P}_T = \G_T^\dagger
\qquad \quad
\text{and}
\qquad \quad
\mathbf{P}_{t-1} = \mathbf{P}_{t} + \mathbf{P}_{t} \x_{t} \transpose{\x}_{t} \mathbf{P}_{t}\,.
\]
Because \texttt{MM} is minimax optimal if the (stringent) conditions
\begin{equation}
\label{eq:cstr}
\forall t \in \{1,\ldots,T\}, \qquad
\sum_{s=1}^{t-1} \, \Bigl| \transpose{\x_s} \mathbf{P}_t^\dagger \x_t \Bigr| \leq 1\,,
\end{equation}
on the feature sequence $\x_1,\dots,\x_T$ are met, a consequence of Theorem~\ref{thm:nlridgeadapted} is that \texttt{MM} also satisfies the regret bound $d B^2 \bigl(1 + \log(1 + T/d) \bigr)$ for those feature sequences. \citet[Theorem~8]{Bartlett2015} showed a regret bound with a leading term of $2 d B^2 \log T$. This closed-form bound actually also held for any sequence of features, not only the ones satisfying~\eqref{eq:cstr}, but it is suboptimal by a multiplicative factor of $2$. See Appendix~\ref{sec:Bartlett} for further technical details.
We do not know whether this suboptimal bound is unavoidable (i.e., is due to the algorithm itself) or whether a different analysis
could lead to a better bound for the \texttt{MM} forecaster on sequences not satisfying~\eqref{eq:cstr}.


\end{rema}

		\begin{rema}
	It is worth to notice that our result holds in a less restrictive setting than beforehand-known features. Indeed, in the definition of the weight vector $\uhat_t$, see Equations~\eqref{eq:nlridgeadapted} and~\eqref{eq:nlridgeadapted-bis}, the only forward information used lies in the regularization term
$\lambda \, \transpose{\ud} \G_T \ud$. Therefore, our algorithm does not need to know the whole sequence of features $\x_1,\dots,\x_T$ in advance: it is enough to know the Gram matrix $\G_T$, in which case our results still hold true. A particular case is when the sequence of features is only known beforehand up to an unknown (and possibly random) permutation, as considered, e.g., by~\citet{kotlowski2017random}.
	\end{rema}

	\begin{proof}
		In order to keep things simple, we will assume here that $\G_T$ is full rank; the proof in the general case can be found in Appendix~\ref{apd:thm2proof}. Then, all matrices $\lambda\, \G_T + \G_t$ are full rank as well.

		The proof of this theorem relies on the bound of the non-linear ridge regression algorithm of Section~\ref{sec:NLVovk}, applied on a modified sequence of features
		\[
		\xt_t = \G_T^{-1/2} \x_t\,,
		\]
		where $\G_T^{-1/2}$ is the inverse square root of the of the symmetric matrix $\G_T$.
		We successively prove the following two inequalities (where we replaced $r_T$ by its value $d$, as $\G_T$ is full rank),
		\begin{eqnarray}
		\label{eqn:firstthm2}
		\sum_{t=1}^T (\y_t - \uhat_t \cdot \x_t)^2
		&\leq \displaystyle{\inf_{\ud \in \Rd} \left\{ \sum_{t=1}^T (\y_t - \ud \cdot \xt_t)^2 +\lambda \norm{\ud}^2 \right\} +  d B^2 \log \left(1+\frac{1}{\lambda}\right)} \\
		\label{eqn:firstthm2-1}
		&\leq \displaystyle{\inf_{\ud \in \Rd} \left\{ \sum_{t=1}^T (\y_t - \ud \cdot \x_t)^2  \right\} + \lambda T B^2 +  d B^2 \log \left(1+\frac{1}{\lambda}\right)}.
		\end{eqnarray}

		\noindent{\emph{Proof of~\eqref{eqn:firstthm2}}.}~~We first show that the strategy~\eqref{eq:nlridge} on the $\xt_t$ leads to
		the same forecasts as the strategy~\eqref{eq:nlridgeadapted} on the original $\x_t$; that is, we show that
		\begin{equation}
        \label{eq:whitening}
		\ut_t \cdot \xt_t = \uhat_t \cdot \x_t\,, \qquad \mbox{where} \qquad
		\ut_t \in \argmin_{\ud \in \Rd} \left\{\sum_{s=1}^{t-1} (\y_s - \ud \cdot \xt_s)^2 + (\ud \cdot \xt_t)^2 + \lambda \norm{\ud}^2 \right\}.
		\end{equation}
		The equality above follows from the definition $\xt_t = \G_T^{-1/2} \x_t$ and the fact that
		$\ut_t = \G_T^{1/2} \uhat_t$. Indeed,
		the closed-form expression~\eqref{eq:Ab1} indicates that
		\[
		\ut_t = \left( \lambda \, \Id + \sum_{s=1}^{t} \xt_s \transpose{\xt_s} \right)^{\!\! -1} \sum_{s=1}^{t-1} y_s \xt_s
		=
		\Bigl(\lambda \Id + \G_T^{-1/2} \G_{t} \G_T^{-1/2} \Bigr)^{-1}
		\G_T^{-1/2}
		\B_{t-1}\,.
		\]
		Now,
		\[
		\Bigl(\lambda \Id + \G_T^{-1/2} \G_{t} \G_T^{-1/2} \Bigr)^{-1}
		=
		\Bigl(\G_T^{-1/2} \bigl(\lambda \G_T +  \G_{t} \bigr) \G_T^{-1/2} \Bigr)^{-1}
		=
		\G_T^{1/2} \bigl(\lambda \G_T +  \G_{t} \bigr)^{-1} \G_T^{1/2}\,,
		\]
		so that
		\[
		\ut_t
		=
		\G_T^{1/2} \bigl(\lambda \G_T +  \G_{t} \bigr)^{-1} \G_T^{1/2}
		\G_T^{-1/2}
		\B_{t-1}
		=
		\G_T^{1/2} \bigl(\lambda \G_T + \G_{t} \bigr)^{-1}
		\B_{t-1}
		=
		\G_T^{1/2} \uhat_t
		\,.
		\]
		We apply the bound of Theorem~\ref{thm:nlridge} on sequences $\xt_1,\ldots,\xt_T \in \R^d$ and $\y_1,\ldots,y_T \in [-B,B]$, to get, for all $\ud \in \Rd$,
		\begin{equation}
		\label{eq:thRNLxt}
		\sum_{t=1}^T (\y_t - \yhat_t)^2 = \sum_{t=1}^T (\y_t - \ut_t \cdot \xt_t)^2 \leq \sum_{t=1}^T (\y_t - \ud \cdot \xt_t)^2 +\lambda \norm{\ud}^2 + B^2 \sum_{k=1}^d \log \! \left(1 + \frac{\eigen_{k}\Bigl(\sum_{t=1}^T \xt_t \transpose{\xt_t}\Bigr)}{\lambda} \right)\,.
		\end{equation}
		The Gram matrix of the $\xt_t$ equals
		\begin{equation}
		\label{eqn:adaptedcovariance}
		\sum_{t=1}^T \xt_t \transpose{\xt_t} = \G_T^{-1/2} \left( \sum_{t=1}^T \x_t \transpose{\x_t} \right) \G_T^{-1/2} = \G_T^{-1/2} \, \G_T \, \G_T^{-1/2} = \Id\,,
		\end{equation}
		so that
		\[
		\sum_{k=1}^d \log \! \left(1 + \frac{\eigen_{k}\Bigl(\sum_{t=1}^T \xt_t \transpose{\xt_t}\Bigr)}{\lambda}\right) = d \log \! \left(1+\frac{1}{\lambda}\right)\,.
		\]
		Taking the infinimum over $\ud$ in $\Rd$ in~\eqref{eq:thRNLxt} concludes the proof of \eqref{eqn:firstthm2}.

		\

		\noindent{\emph{Proof of}~\eqref{eqn:firstthm2-1}.}~~We bound \vspace{-.5cm}
		\[
		\qquad \inf_{\ud \in \Rd} \left\{ \sum_{t=1}^T (\y_t - \ud \cdot \xt_t)^2 +\lambda \norm{\ud}^2 \right\}\,, \vspace{.25cm}
		\]
		by evaluating it at $\smash{\ud^{\star} \in \displaystyle{\argmin_{\ud \in \Rd} \bigg\{ \sum_{t=1}^T (\y_t - \ud \cdot \xt_t)^2 \bigg\}}}$, which is a singleton with closed-form expression \vspace{.2cm}
		\[
		\ud^{\star}
		= \Biggl(\sum_{t=1}^{T} \xt_t \transpose{\xt}_t \Biggr)^{-1} \Biggl(\sum_{t=1}^{T} \y_t \xt_t \Biggr)
		= \G_T^{-1/2} \B_T\,,
		\]
		where we used \eqref{eqn:adaptedcovariance} and
        where $\B_{T}$ was defined in~\eqref{eq:Ab1}.
		We first bound $\norm{\ud^\star}^2$. By denoting
		\[
		\X_T = \bigl[ \begin{array}{ccc}
		\x_1 &
		\cdots &
		\x_T
		\end{array}
		\bigr]
		\qquad \quad
		\text{and}
		\qquad \quad
		\Y_T = \left[ \begin{array}{c}
		\y_1 \\
		\vdots \\
		\y_T
		\end{array} \right],
		\]
		which are respectively, a $d \times T$ and a $T \times 1$ matrix,
		we have
		\begin{equation}
        \label{eq:elementaryortho}
		\ud^\star = \G_T^{-1/2} \X_T \Y_T\,, \qquad \mbox{thus} \qquad
		\norm{\ud^\star}^2 = \transpose{\Y}_T \transpose{\X}_T \G_T^{-1} \X_T \Y_T\,.
		\end{equation}
		Noting that $\transpose{\X}_T \G_T^{-1} \X_T$ is an orthogonal projection (on the image of $\transpose{\X}_T$)
		entails the inequalities $\norm{\ud^\star}^2 \leq \norm{\Y_T}^2 \leq T B^2$. Putting all elements together, we proved so far
		\[
		\inf_{\ud \in \Rd} \left\{ \sum_{t=1}^T (\y_t - \ud \cdot \xt_t)^2 +\lambda \norm{\ud}^2 \right\}
		\leq
		\inf_{\ud \in \Rd} \left\{ \sum_{t=1}^T (\y_t - \ud \cdot \xt_t)^2 \right\} + \lambda T B^2\,.
		\]
		We conclude the proof of \eqref{eqn:firstthm2-1} by a change of dummy variable
		$\vd = \G_T^{1/2} \ud$ and the fact that since $\G_T$ is full rank, its image is $\Rd$:
		\begin{equation*}
		\inf_{\ud \in \Rd} \left\{\sum_{t=1}^T (\y_t - \ud \cdot \xt_t)^2\right\}
		=\inf_{\ud \in \Rd} \left\{\sum_{t=1}^T (\y_t - \G_T^{1/2} \ud \cdot \x_t)^2\right\}
		=\inf_{\vd \in \Rd} \left\{\sum_{t=1}^T (\y_t - \vd \cdot \x_t)^2\right\}\,.
		\end{equation*}
		\vspace{-1cm}

	\end{proof}

\section{Sequentially revealed features / New result}
\label{sec:sequentialfeatures}

In this section we do not assume that the features are known beforehand (i.e., unlike in the previous section) and yet exhibit a simple forecaster with a regret bound of $d B^2 \log T + \cO_T(1)$ holding uniformly over $\Rd$. Perhaps unexpectedly, the solution that we propose is just to remove the regularization term $\lambda \norm{\ud}_{\G_T}^2$ in~\eqref{eq:nlridgeadapted}, which cannot be computed in advance.
This amounts to considering the standard non-linear ridge regression algorithm~\eqref{eq:nlridge} with a regularization factor $\lambda = 0$.
The reason why this is a natural choice is explained in Remark~\ref{rk:why} below.

Thus, weights vectors defined as in Equations~\eqref{eq:nlridge} or~\eqref{eq:nlridgeadapted} with regularization parameter $\lambda = 0$ are picked:
$\uhat_1 = (0,\ldots,0)^{\transp}$ and for $t \geq 2$,
\begin{equation}
\label{eq:ridgeNL0}
	\uhat_t \in \argmin_{\ud \in \Rd} \left\{ \sum_{s=1}^{t-1} (\y_s - \ud \cdot \x_s)^2
	+ (\ud \cdot \x_t)^2 \right\},
\qquad \mbox{hence} \qquad \uhat_t = \G_t^\dagger \B_{t-1}\,,
\end{equation}
where the closed-form expression corresponds to~\eqref{eqn:nlridgeadaptedcf}.
It then predicts $\yhat_t = \uhat_t \cdot \x_t$ (and as already indicated after~\eqref{eq:nlridge},
no clipping can take place as $B$ is unknown to the learner).

Note that no parameter requires to be tuned in this case, which can be a true relief.

\begin{rema}
The traditional bound for the
non-linear ridge regression forecaster blows up
when the regularization parameter is set as $\lambda = 0$
(see Section~\ref{sec:NLVovk}) but, perhaps surprisingly,
an ad hoc analysis could be performed here---see Theorem~\ref{thm:nlridgesequenceadapted}.
It provides a new understanding of this well-known non-linear regression algorithm:
the regularization term $\lambda \norm{\ud}^2$ in its defining equation~\eqref{eq:nlridge}
is not so useful, while the seemingly harmless regularization term $(\ud \cdot \x_t)^2$ therein
is crucial.
\end{rema}

The theorem below follows from a combination of arguments all already present in the literature,
namely, \cite[Theorem~3.2]{FoWa03}, \citet[Lemma~D.1]{CBCG05}, \citet[Theorem~4 of Appendix~D]{Luo16},
with a slightly more careful analysis at only one small point in the proof of the latter;
see details in Appendix~\ref{apd:nlridgeadaptatedproof}.
The proof is actually based on the proof of Theorem~\ref{thm:nlridge}
but requires adaptations to account for the fact that $\uhat_t$
is defined in~\eqref{eq:ridgeNL0} in terms of a possibly non-invertible matrix $\G_t$.
There are strong links between the results of Theorem~\ref{thm:nlridgesequenceadapted}
and Theorem~\ref{thm:nlridge}; see Remark~\ref{rk:finitelymany} below.

The result of Theorem~\ref{thm:nlridgesequenceadapted} is not that
straightforward, and in particular, some tricks
that were suggested to us when presenting this work, e.g.,
neglecting finitely many rounds till the Gram matrix is full rank
(if this ever happens), would probably work but would lead to an even larger
constant term. Generally speaking, neglecting finitely many rounds may
have important side-effects, see an illustration in Remark~\ref{rk:finitelymany}.

	\begin{theorem}
		\label{thm:nlridgesequenceadapted} For all $T \geq 1$, for all sequences $\x_1,\ldots,\x_T \in \R^d$ and all $\y_1,\ldots,y_T \in [-B,B]$, the non-linear ridge regression algorithm with $\lambda = 0$ as in~\eqref{eq:ridgeNL0} achieves the uniform regret bound
        \[
        \sup_{\ud \in \R^d}
		\regret_T(\ud)
\leq B^2 \sum_{t=1}^T \transpose{\x}_t \G_t^\dagger \x_t \leq
B^2 \sum_{k=1}^{r_T} \log \bigl(\eigen_k(\G_T) \bigr) + B^2 \sum_{t \in [\![1, T]\!] \cap \jump} \log \! \left( \frac{1}{\eigen_{r_t}(\G_{t})} \right) + r_T B^2
        \]
where $r_t$ and $\lambda_k$ are defined in Notation~\ref{not1}, and where the set $\jump$ contains $r_T$ rounds,
given by the smallest $s \geq 1$ such that $\x_s$ is not null,
and all the $s \geq 2$ for which $\rank(\G_{s-1}) \ne \rank(\G_s)$.
	\end{theorem}

We recall in Appendix~\ref{sec:GramMatr} that $\rank(\G_t)$ is a non-decreasing sequence,
with increments of~$1$, hence the claimed cardinality $r_T$ of $\jump$, and the fact
that $\eigen_{r_t}(\G_{t}) > 0$ for all $t \in \jump$.

Note that the regret bound obtained is scale invariant, which is natural and was expected,
as the forecaster also is; to see why this is the case, note that it only involves
quantities~$\eigen_k(\G_T)/\eigen_{r_t}(\G_{t})$.

The same (standard) arguments as the ones at the end of the proof of Corollary~\ref{cor:RNL-standard}
show the following consequence of this bound (which is scale invariant as far as
multiplications of the features by scalar factors only are concerned):
for all $X > 0$, for all sequences $\x_1,\x_2,\ldots$ of features with $\norm{\x_t} \leq X$,
\[
\sup_{\ud \in \R^d} \regret_T(\ud)
\leq d B^2 \log T + \underbrace{d B^2 + B^2 \sum_{t \in [\![1, T]\!] \cap \jump} \log \! \left( \frac{X^2}{\eigen_{r_t}(\G_{t})} \right)}_{\text{this is our $\cO_T(1)$ here}}\,.
\]
Note that the $\cO_T(1)$ term stops increasing once the matrix $\G_T$ is full rank:
then, for rounds $T' \geq T$, only the leading term increases to $d B^2 \log T'$.
But this $\cO_T(1)$ can admittedly be large and blows up as all sequences of feature vectors
are considered, just as in the bound of Corollary~\ref{cor:RNL-standard}. The dependency
on the eigenvalues is however slightly improved to a logarithmic one, here.

The bound of Theorem~\ref{thm:nlridgesequenceadapted} thus still
remains somewhat weak, hence our main open question.

\subsection{Open question---Double uniformity over $\R^d$: for the $\ud$ and for the $\x_t$}
\label{sec:openquestion}

For the time being, no $d B^2 \ln(T) + \cO_T(1)$ regret bound simultaneously uniform over
all comparison vectors $\ud \in \R^d$ and over all features $\x_t$ with $\norm{\x_t} \leq X$ and bounded observations $y_t \in [-B,B]$
is provided in the case of sequentially revealed features (what we called worst-case uniform regret bounds). Indeed,
the bound of Theorem~\ref{thm:nlridge} is not uniform over the comparison vectors $\ud \in \R^d$.
The bound of Corollary~\ref{cor:RNL-standard} is of order $2 d B^2 \ln(T)$
(and is not uniform over even bounded feature vectors).
The bound of Theorem~\ref{thm:nlridgeadapted} enjoys the double uniformity and is of proper order,
but only holds for beforehand-known features.
The bound of Theorem~\ref{thm:nlridgesequenceadapted} is uniform over the comparison vectors $\ud \in \R^d$,
is of proper order $d B^2 \ln(T)$ and holds in the sequential case, but is not uniform over
bounded features $\x_t$ (its remainder term can be large).

The lower bound of Theorem~\ref{thm:lowerbound} (and earlier lower bounds by
\citealp{Vovk01} and \citealp{TW00}) are proved in the case of feature vectors that are initially revealed to the regression strategy.
The open question is therefore whether we can improve the lower bound and make it larger for strategies that only discover the features on the fly, or if a doubly uniform regret upper bound of $d B^2 \ln(T) + \cO_T(1)$ over the $\ud$ and the $\x_t,\,y_t$ is also possible in the case of sequentially revealed features. For the lower bound, it seems that choosing random feature vectors that are independent over time might not be a good idea, since the final normalized Gram matrix $\G_T/T$ may be concentrated around its expectation $\G$, and the regression strategy might use the possibly known $\G$ to transform the features $\x_t$ as in Theorem~\ref{thm:nlridgeadapted}. Instead, choosing random features $\x_t$ that are dependent over time might make the task of predicting the final Gram matrix $\G_T$ virtually impossible, and might help to improve the lower bound. Alternatively, we could construct features $\x_t$ in a truly sequential manner, as functions of the strategy's past predictions, so as to annoy the regression strategy.

\subsection{Some further technical remarks}

We provide details on two claims issued above.

\begin{rematitre}{Links between Theorem~\ref{thm:nlridgesequenceadapted} and Theorem~\ref{thm:nlridge}}
\label{rk:finitelymany}
Assume that
we use the non-linear ridge regression algorithm with $\lambda = 0$ as in~\eqref{eq:ridgeNL0}
but feed it first with $d$ warm-up feature vectors
$\x_{-t} = (0,\ldots,0, \sqrt{\lambda}, 0, \ldots, 0)$, where the $\sqrt{\lambda}$ is
in position $t \in \{1,\ldots,d\}$, and that the observations are $y_{-t} = 0$.
Then for each $\ud \in \R^d$, a cumulative loss of $\lambda \norm{\ud}^2$ is suffered,
and to neglect these $d$ additional rounds in the regret bound obtained by
Theorem~\ref{thm:nlridgesequenceadapted}, we need to add a $\lambda \norm{\ud}^2$
term to it. As all terms corresponding to the new eigenvalues introduced
$\eigen_{r_t}(\G_{t})$ are equal to $\lambda$, given the choice of these
warm-up features, we are thus essentially back to the bound of Theorem~\ref{thm:nlridge}.
\end{rematitre}

\begin{rematitre}{How we came up with the forecaster~\eqref{eq:ridgeNL0}}
\label{rk:why}
A natural attempt to transform the forecaster~\eqref{eq:nlridgeadapted} designed
for the case of beforehand-known features into a fully sequential algorithm
is to replace the matrix $\G_T$ that is unknown at the beginning of round $t$ by its sequential
estimate $\G_{t}$ and to regularize at time $t$
with $(\ud \cdot \x_t)^2 + \lambda \norm{\ud}_{\G_{t}}^2$ instead of $(\ud \cdot \x_t)^2 + \lambda \norm{\ud}_{\G_T}^2$
as in~\eqref{eq:nlridgeadapted}.
However, in this case, the closed-form expression for the vector $\uhat_t$ is
$\smash{\uhat_t = \G_t^\dagger \B_{t-1}/(1+\lambda)}$, that is, the $\lambda$ only acts as a
multiplicative bias to the vector otherwise considered in~\eqref{eq:ridgeNL0}.
The analysis we followed led to a regret bound increasing in $\lambda$, so that we finally picked $\lambda = 0$
and ended up with our non-linear ridge regression algorithm with $\lambda = 0$ as in~\eqref{eq:ridgeNL0}.
\end{rematitre}

\appendix
\section{Details on the proof of Theorem~\ref{thm:lowerbound}}
\label{sec:vT}

\subsection*{Details on getting~\eqref{eq:prendreesp}}

By exchanging an expectation and an infimum,
the expectation of the uniform regret of any fixed forecaster considered can be bounded as
    \begin{equation}
    \label{eq:bdint01}
    \E_{\thetav^\star} \! \left[ \sup_{\ud \in \R^d} \regret_T(\ud) \right]
    \geq \sum_{t=1}^T \E_{\thetav^\star} \Bigl[ (Y_t - \yhat_t)^2 \Bigr] -
    \inf_{\ud \in \R^d} \sum_{t=1}^T \E_{\thetav^\star} \Bigl[ (Y_t - \ud \cdot \ev_{J_t})^2 \Bigr]\,.
    \end{equation}
    Since $\yhat_t$ is measurable w.r.t.\ $\cF_{t-1}$, the $\sigma$--algebra generated by the information available at
    the beginning of round $t$, namely, $J_1,\ldots,J_T$ and $Y_1,\ldots,Y_{t-1}$, and since
    $Y_t$ is distributed, conditionally on $\cF_{t-1}$  according to a Bernoulli distribution with parameter $\theta^\star_{J_t}$,
    a conditional bias--variance decomposition yields
    \begin{align*}
    \E_{\thetav^\star} \Bigl[ (\yhat_t - Y_t)^2 \, \big| \, \cF_{t-1} \Bigr] & =
    (\yhat_t - \theta^\star_{J_t})^2
    + \E_{\thetav^\star} \Bigl[ (Y_t - \theta^\star_{J_t})^2 \, \big| \, \cF_{t-1} \Bigr] \\
    & = \bigl(\hatu_{J_t,t} - \theta^\star_{J_t} \bigr)^2
    + \theta^\star_{J_t} \bigl( 1-\theta^\star_{J_t} \bigr)\,,
    \end{align*}
    where we also used that by construction, $\yhat_t = \uhat_t \cdot \ev_{J_t} = \hatu_{J_t,t}$.
    Similarly, for all $\ud \in \R^d$,
    \[
    \E_{\thetav^\star} \Bigl[ (Y_t - \ud \cdot \ev_{J_t})^2 \, \big| \, \cF_{t-1} \Bigr] =
    \bigl( u_{J_t} - \theta^\star_{J_t} \bigr)^2
    + \theta^\star_{J_t} \bigl( 1-\theta^\star_{J_t} \bigr)\,.
    \]
    By the tower rule and since the variance terms $\theta^\star_{J_t} \bigl( 1-\theta^\star_{J_t} \bigr)$
    cancel out, we thus proved that
    \begin{align*}
    \E_{\thetav^\star} \! \left[ \sup_{\ud \in \R^d} \regret_T(\ud) \right]
    & \geq \sum_{t=1}^T \E_{\thetav^\star} \Bigl[ (\yhat_t - Y_t)^2 \Bigr] -
    \inf_{\ud \in \R^d} \sum_{t=1}^T \E_{\thetav^\star} \Bigl[ (Y_t - \ud \cdot \ev_{J_t})^2 \Bigr] \\
    & = \sum_{t=1}^T \E_{\thetav^\star} \Bigl[ \bigl( \hatu_{J_t,t} - \theta^\star_{J_t} \bigr)^2 \Bigr]
    - \inf_{\ud \in \R^d} \sum_{t=1}^T  \E_{\thetav^\star} \Bigl[ \bigl( u_{J_t} - \theta^\star_{J_t} \bigr)^2 \Bigr] \\
    & = \sum_{t=1}^T \E_{\thetav^\star} \Bigl[ (\hatu_{J_t,t} - \theta^\star_{J_t})^2 \Bigr]\,.
    \end{align*}
    Now, by resorting to the tower rule again, integrating over $J_t$ conditionally on $Y_1,\ldots,Y_{t-1}$
    and $J_1,\ldots,J_{t-1},J_{t+1},\ldots,J_T$, we get
    \begin{equation}
    \label{eq:preuveborneinfpartielle}
    \E_{\thetav^\star} \! \left[ \sup_{\ud \in \R^d} \regret_T(\ud) \right]
    \geq \sum_{t=1}^T \E_{\thetav^\star} \Bigl[ (\hatu_{J_t,t} - \theta^\star_{J_t})^2 \Bigr]
    = \sum_{t=1}^T \frac{1}{d} \, \E_{\thetav^\star} \Bigl[ \bigl\Arrowvert \uhat_t - \theta^\star \bigr\Arrowvert^2_2 \Bigr]\,.
    \end{equation}
    We now show that each term in the sum is larger than something of the order of $d/t$.
    This order of magnitude $d/t$ is the parametric rate of optimal estimation; indeed, due to the randomness of the $J_s$, over $t$ periods, each component is used about
    $t/d$ times, while the rate of convergence in quadratic error of any $d$--dimensional estimator
    based on $\tau = t/d$ unbiased i.i.d.\ observations is at best $d/\tau = d^2/t$. Taking
    into account the $1/d$ factor gets us the claimed $d/t$ rate.
    The next steps (based on the van Trees inequality) transform this intuition into formal statements.

\subsection*{Conclusion of the proof, given the application of the van Trees inequality}

    We resume at~\eqref{eq:preuveborneinfpartielle} and consider a prior $\pi$ over the $\theta^\star \in [0,1]^d$.
    Since an expectation is always smaller than a supremum, we have first, given the defining equation~\eqref{eq:defregretstar} of
    $\regret^\star_{T,\,[0,1]}$,
    \begin{align*}
    \regret^\star_{T,\,[0,1]} & \geq \inf_{\text{forecasters}} \
    \bigintsss_{[0,1]^d} \E_{\theta^\star} \! \left[ \sup_{\ud \in \R^d} \regret_T(\ud) \right] \d\pi(\theta^\star) \\
    & \geq \inf_{\text{forecasters}} \ \sum_{t=1}^T
    \bigintsss_{[0,1]^d} \frac{1}{d} \, \E_{\thetav^\star} \Bigl[ \bigl\Arrowvert \uhat_t - \theta^\star \bigr\Arrowvert^2_2 \Bigr] \d\pi(\theta^\star)\,,
    \end{align*}
    where the second inequality follows
    by mixing both sides of~\eqref{eq:preuveborneinfpartielle} according to $\pi$.
    Now, an immediate application of the (multi-dimentional) van Trees inequality with a
    Beta($\alpha,\alpha$) prior $\pi$ shows that for all forecasters, all $t \geq 1$ and $\alpha \geq 3$,
    \[
    \bigintsss_{[0,1]^d} \frac{1}{d} \, \E_{\theta^\star} \Bigl[ \bigl\Arrowvert \uhat_t - \theta^\star \bigr\Arrowvert^2_2 \Bigr] \d\pi(\theta^\star) \geq \frac{d}{4(t-1) + 2(t-1)/(\alpha-1) + 16 d \alpha}\,,
    \]
    see Lemma~\ref{lm:vT} below. We thus proved
    \begin{align*}
    \nonumber
    & \regret^\star_{T,\,[0,1]} \geq \sum_{t=0}^{T-1} \frac{d}{\bigl( 4 + 2/(\alpha-1) \bigr) t + 16 d \alpha}
    \geq d \bigintsss_0^{T} \frac{1}{\bigl( 4 + 2/(\alpha-1) \bigr) t + 16 d \alpha} \d t \\
    & = \frac{d}{4 + 2/(\alpha-1)} \ln \frac{\bigl( 4 + 2/(\alpha-1) \bigr) T + {16 d \alpha}}{16 d \alpha} \\
    & \geq \frac{d}{4 + 2/(\alpha-1)} \ln \frac{ 4 T}{16 d \alpha} \\
    & = \frac{d}{4 + 2/(\alpha-1)} \bigl(\ln T -  \ln (4d\alpha)\bigr)\,,
    \end{align*}
    which we lower bound in a crude way by resorting to $1/(1+u) \geq 1-u$
    and by taking $\alpha$ such that $\alpha-1 = \ln T$; this is where our condition $T \geq 8 > \e^2$ is used, to ensure that $\alpha \geq 3$. We also use that since $T \geq \e^2$, we have $1 \leq (\ln T)/2$
    thus $4d\alpha \leq 4d(1+\ln T) \leq 6d\log T$.
    We get
    \begin{align}
    \nonumber
    \regret^\star_{T,\,[0,1]} & \geq \frac{d}{4} \biggl( \underbrace{1 - \frac{1}{2(\alpha-1)}}_{\geq 0} \biggr)
    \bigl( \ln T - \ln( 6d\log T) \bigr) \\
    \label{eq:bdint01-bis}
    & \geq \frac{d}{4} \biggl( 1 - \frac{1}{2 \ln T} \biggr) ( \ln T  - \ln(6d) - \ln \ln T)
    \geq \frac{d}{4} \bigl( \ln T - (3+\ln d) - \ln \ln T \bigr)\,.
    \end{align}
    The factor $3$ above corresponds to $1/2 + \ln 6 \leq 3$.
    So, we covered the case of $\regret^\star_{T,\,[0,1]}$ and now turn
    to $\regret^\star_{T,\,[-B,B]}$ for a general $B > 0$.

\subsection*{Going from $\regret^\star_{T,\,[0,1]}$ to $\regret^\star_{T,\,[-B,B]}$}

    To get a lower bound of exact order $d \ln T$,
    that is, to get rid of the annoying multiplicative factor of $1/4$,
    we proceed as follows. With the notation above, $Z_t = 2B(Y_t - 1/2)$ lies in $[-B,B]$.
    Denoting by $\hat z_t$ the forecasts output by a given forecaster sequentially fed with the
    $(Z_s,\ev_{J_s})$, we have
    \[
    (\hat z_t - Z_t)^2 = 4B^2 (\yhat_t - Y_t)^2 \qquad \mbox{where the} \qquad
    \yhat_t = \frac{\hat z_t + 1/2}{2B}
    \]
    also correspond to predictions output by a legitimate forecaster, and
    \[ \!\!\!\!
    \inf_{\vd \in \R^d} \sum_{t=1}^T \E \Bigl[ (Z_t - \vd \cdot \ev_{J_t})^2 \Bigr] =
    4B^2 \inf_{\vd \in \R^d} \sum_{t=1}^T \E \Biggl[ \biggl(Y_t - \frac{1}{2} - \frac{\vd \cdot \ev_{J_t}}{2B} \biggr)^{\!\! 2} \Biggr]
    =  4B^2 \inf_{\ud \in \R^d} \sum_{t=1}^T \E \Bigl[ (Y_t - \ud \cdot \ev_{J_t} )^2 \Bigr]
    \]
    by considering the transformation $\vd \leftrightarrow \ud$ given by
    $u_j = v_j/(2B) - 1/2$. (We use here that the sum of the components of the
    $\ev_{J_t}$ equal $1$.)
    We thus showed that $\regret^\star_{T,\,[-B,B]}$ is larger
    than $4B^2$ times the lower bound~\eqref{eq:bdint01-bis} exhibited on~\eqref{eq:bdint01},
    which concludes the proof.

\subsection*{Details on the application of the van Trees inequality}

The van Trees inequality is a Bayesian version of the {C}ram{\'e}r-{R}ao bound,
but holding for any estimator (not only the unbiased ones); see \citet[Section~4]{GiLe95}
for a multivariate statement (and refer to \citealp{vT68} for its first statement). \medskip

Recall that we denoted above by $\mathcal{P}_{\theta^\star}$ the distribution of the sequence of pairs
$(J_t,Y_t)$, with $1 \leq t \leq T$, considered in Section~\ref{sec:NLVovk-COLS} for a given $\theta^\star \in [0,1]^d$.
We also considered the family $\mathcal{P}$
of these distributions and thus, for clarity, indexed all expectations $\E$
by the underlying parameter $\theta^\star$ at hand.
We introduce a product of independent Beta($\alpha,\alpha$) distributions
as a prior $\pi$ on the $\theta^\star \in [0,1]^d$; its density
with respect to the Lebesgue measure equals
\[
\beta_{\alpha,\alpha}^{(d)} (t_1,\ldots,t_d) \longmapsto
\beta_{\alpha,\alpha}(t_1) \cdots \beta_{\alpha,\alpha}(t_d)\,,
\qquad \mbox{where} \qquad
\beta_{\alpha,\alpha} : t \mapsto \frac{\Gamma(2\alpha)}{\bigl( \Gamma(\alpha) \bigr)^2} \,
t^{\alpha-1}(1-t)^{\alpha-1}\,.
\]
The reason why Beta distributions are considered is
because of the form of the Fisher information of the $\mathcal{P}$ family,
see calculations~\eqref{eq:whyBeta} below. \medskip

The multivariate van Trees inequality ensures that for all estimators $\uhat_t$, that is,
for all random variables which are measurable functions of $J_1,\ldots,J_T$ and $Y_1,\ldots,Y_{t-1}$,
we have
\begin{equation}
\label{eq:vT}
\bigintsss_{[0,1]^d}
\E_{\theta^\star} \Bigl[ \bigl\Arrowvert \uhat_t - \theta^\star \bigr\Arrowvert^2_2 \Bigr]
\,\, \beta_{\alpha,\alpha}^{(d)} (\theta^\star)
\, \d\theta^\star
\geq \frac{(\Tr \Id)^2}{\displaystyle{\Tr \cI\bigl( \beta_{\alpha,\alpha}^{(d)} \bigr) +
\int_{[0,1]^d} \bigl( \Tr \cI(\theta^\star) \bigr) \, \beta_{\alpha,\alpha}^{(d)} (\theta^\star)
\, \d\theta^\star}}\,,
\end{equation}
where $\d\theta^\star$ denotes the integration w.r.t.\ Lebesgue measure, $\Tr$ is the trace operator,
$\cI(\theta^\star)$ stands for the Fisher information of the family $\mathcal{P}$ at
$\theta^\star$, see~\eqref{eq:fisher},
while each component $(i,i)$ of the other matrix in the denominator is given by
\[
\cI\bigl( \beta_{\alpha,\alpha}^{(d)} \bigr)_{i,i}
\eqdef \bigintss_{[0,1]^d} \biggl( \frac{\partial \beta_{\alpha,\alpha}^{(d)}}{\partial \theta^\star_i} (\theta^\star)
\biggr)^{\!\! 2} \frac{1}{\beta_{\alpha,\alpha}^{(d)}(\theta^\star)}
\, \d\theta^\star\,,
\]
which may equal $+\infty$ (in which case the lower bound is void).
There are conditions for the inequality to be satisfied, we detail them in the proof
of the lemma below.

\begin{lemma}
\label{lm:vT}
When the family $\mathcal{P}$ is equipped with a prior given by a
product of independent Beta($\alpha,\alpha$) distributions, where $\alpha \geq 3$,
it follows from the van Trees inequality and from simple calculations that
\[
\bigintsss_{[0,1]^d}
\E_{\theta^\star} \Bigl[ \bigl\Arrowvert \uhat_t - \theta^\star \bigr\Arrowvert^2_2 \Bigr]
\,\, \beta_{\alpha,\alpha}^{(d)} (\theta^\star)
\, \d\theta^\star
\geq \frac{d^2}{16 d \alpha + 4(t-1) + 2(t-1)/(\alpha-1)}\,.
\]
\end{lemma}

\begin{proof}
We denote by
\[
f_{\theta^\star} : (j_1,\ldots,j_T,\,y_1,\ldots,y_{t-1}) \in \{1,\ldots,d\}^T \times \{0,1\}^{t-1} \longmapsto
\frac{1}{d^T} \prod_{s=1}^{t-1} {\theta^\star_{j_s}}^{y_s} (1-\theta^\star_{j_s})^{1-y_s}
\]
the density of $\mathcal{P}_{\theta^\star}$ w.r.t.\ to the counting measure $\mu$ on
$\{1,\ldots,d\}^T \times \{0,1\}^{t-1}$.

The sufficient conditions of \citet[Section~4]{GiLe95} for~\eqref{eq:vT} are met,
since on the one hand
$\beta_{\alpha,\alpha}^{(d)}$ is $C^1$--smooth, vanishes on the border of $[0,1]^d$,
and is positive on its interior, while on the other hand,
$\theta^\star \mapsto f_{\theta^\star}(j_1,\ldots,j_T,\,y_1,\ldots,y_{t-1})$ is $C^1$--smooth
for all $(j_1,\ldots,j_T,\,y_1,\ldots,y_{t-1})$, with,
for all $i \in \{1,\ldots,d\}$,
\[
L_i(\theta^\star) = \frac{\partial}{\partial \theta^\star_i} \ln f_{\theta^\star}(J_1,\ldots,J_T,\,Y_1,\ldots,Y_{t-1})
= \sum_{s=1}^{t-1} \left( \frac{Y_s}{\theta^\star_{J_s}} - \frac{1-Y_s}{1-\theta^\star_{J_s}} \right) \ind{\{J_s=i\}}
\]
being square integrable, so that the Fisher information matrix $\cI(\theta^\star)$ of the $\mathcal{P}$ model
at $\theta^\star$ exists and has a component $(i,i)$ given by
\begin{align}
\nonumber
\cI(\theta^\star)_{i,i} \eqdef
\E_{\theta^\star} \Bigl[ L_i(\theta^\star)^2 \Bigr]
& = (t-1) \,\,
\E_{\theta^\star} \! \Biggl[\biggl( \frac{Y_1}{\theta^\star_{J_1}} - \frac{1-Y_1}{1-\theta^\star_{J_1}} \biggr)^{\!\! 2} \ind{\{J_1=i\}} \Biggr] \\
\label{eq:fisher}
& = \frac{t-1}{d} \biggl( \frac{1}{\theta^\star_i} + \frac{1}{1-\theta^\star_i} \biggr)
= \frac{t-1}{d \, \theta^\star_i (1-\theta^\star_i)}\,,
\end{align}
and therefore, is such that $\theta^\star \mapsto \sqrt{\cI(\theta^\star)}$ is locally integrable
w.r.t.\ the Lebesgue measure. The second inequality in~\eqref{eq:fisher} is because $L_i(\theta^\star)$
is a sum of $t-1$ centered, independent and identically distributed variables,
while the third inequality is obtained by the tower rule, by first taking the conditional expectation
with respect to $J_1$.

We now compute all elements of the denominator of~\eqref{eq:vT}.
First, by symmetry and then by substituting~\eqref{eq:fisher},
\begin{align}
\nonumber
\lefteqn{\int_{[0,1]^d} \bigl( \Tr \cI(\theta^\star) \bigr) \, \beta_{\alpha,\alpha}^{(d)} (\theta^\star)
\, \d\theta^\star} \\
\nonumber
& = d \, \int_{[0,1]^d} \cI(\theta^\star)_{1,1} \, \beta_{\alpha,\alpha}^{(d)} (\theta^\star)
\, \d\theta^\star \\
\label{eq:whyBeta}
& = d \bigintsss_{[0,1]^d} \frac{t-1}{d \, \theta^\star_1 (1-\theta^\star_1)} \,\,
\frac{\Gamma(2\alpha)}{\bigl( \Gamma(\alpha) \bigr)^2} \, (\theta^\star_1)^{\alpha-1} (1-\theta^\star_1)^{\alpha-1}
\,\, \beta_{\alpha,\alpha}(\theta^\star_2) \, \cdots \,
\beta_{\alpha,\alpha}(\theta^\star_d)
\, \d\theta^\star \\
\nonumber
& = (t-1) \, \frac{\Gamma(2\alpha)}{\bigl( \Gamma(\alpha) \bigr)^2} \,
\int_{[0,1]} z^{\alpha-2} (1-z)^{\alpha-2} \d z
= (t-1) \, \frac{\Gamma(2\alpha)}{\bigl( \Gamma(\alpha) \bigr)^2}
\frac{\bigl( \Gamma(\alpha-1) \bigr)^2}{\Gamma\bigl(2(\alpha-1)\bigr)}\,,
\end{align}
where we used the expression of the density of the Beta($\alpha-1,\alpha-1$)
distribution for the last equality.
Using that $x \, \Gamma(x) = \Gamma(x+1)$ for all real numbers $x > 0$,
we finally get
\[
\int_{[0,1]^d} \bigl( \Tr \cI(\theta^\star) \bigr) \, \beta_{\alpha,\alpha}^{(d)} (\theta^\star)
\, \d\theta^\star
= \frac{(2\alpha-1)(2\alpha-2)}{(\alpha-1)^2} \, (t-1) = \frac{4\alpha - 2}{\alpha-1} \, (t-1)
= 4 (t-1) + \frac{2(t-1)}{\alpha-1}\,.
\]

Second, as far as the
$\Tr \cI\bigl( \beta_{\alpha,\alpha}^{(d)} \bigr)$ in~\eqref{eq:vT} is concerned,
because $\beta_{\alpha,\alpha}^{(d)}$ is a product of univariate distributions,
\[
\cI\bigl( \beta_{\alpha,\alpha}^{(d)} \bigr)_{i,i}
= \bigintss_{[0,1]^d} \biggl( \frac{\partial \beta_{\alpha,\alpha}^{(d)}}{\partial \theta^\star_i} (\theta^\star)
\biggr)^{\!\! 2} \frac{1}{\beta_{\alpha,\alpha}^{(d)}(\theta^\star)}
\, \d\theta^\star
= \bigintss_{[0,1]} \biggl( \frac{\partial \beta_{\alpha,\alpha}}{\partial z} (z) \biggr)^{\!\! 2} \frac{1}{\beta_{\alpha,\alpha}}(z) \, \d z\,,
\]
so that $\Tr \cI\bigl( \beta_{\alpha,\alpha}^{(d)} \bigr)$ equals $d$ times this value, that is,
$d$ times
\begin{align*}
& \bigintss_{[0,1]} \frac{\Gamma(2\alpha)}{\bigl( \Gamma(\alpha) \bigr)^2} \,
\frac{\bigl( (\alpha-1) \, z^{\alpha-2}(1-z)^{\alpha-1} - (\alpha-1) \, z^{\alpha-1}(1-z)^{\alpha-2}
\bigr)^2}{z^{\alpha-1}(1-z)^{\alpha-1}} \,\d z \\
& = \frac{(\alpha-1)^2 \, \Gamma(2\alpha)}{\bigl( \Gamma(\alpha) \bigr)^2} \int_{[0,1]}
(1-2z)^2 \,\, z^{\alpha-3}(1-z)^{\alpha-3} \d z \\
& = \frac{(\alpha-1)^2 \, \Gamma(2\alpha)}{\bigl( \Gamma(\alpha) \bigr)^2} \,
\frac{\bigl( \Gamma(\alpha-2) \bigr)^2}{\Gamma\bigl(2(\alpha-2)\bigr)} \, \E\Bigl[ (1-2 Z_{\alpha-2})^2 \Bigr]
= \frac{(\alpha-1)^2 \, \Gamma(2\alpha)}{\bigl( \Gamma(\alpha) \bigr)^2} \,
\frac{\bigl( \Gamma(\alpha-2) \bigr)^2}{\Gamma\bigl(2(\alpha-2)\bigr)} \, 4\,\mathrm{Var}(Z_{\alpha-2})
\end{align*}
where $Z_{\alpha-2}$ is a random variable following the Beta($\alpha-2,\alpha-2$) distribution; its expectation
equals indeed $\E[Z_{\alpha-2}] = 1/2$ by symmetry of the distribution w.r.t.\ $1/2$, so that
\[
\E\Bigl[ (1-2 Z_{\alpha-2})^2 \Bigr] = 4 \,\, \E\Bigl[ (1/2- Z_{\alpha-2})^2 \Bigr] = 4\,\mathrm{Var}(Z_{\alpha-2}) \qquad \mbox{where} \qquad \mathrm{Var}(Z_{\alpha-2}) = \frac{1}{4(2\alpha-3)}
\]
by a classical formula. Collecting all elements together and
using again that $x \, \Gamma(x) = \Gamma(x+1)$ for all real numbers $x > 0$, we get
\[
\Tr \cI\bigl( \beta_{\alpha,\alpha}^{(d)} \bigr) =
d \,\,
\underbrace{\frac{(\alpha-1)^2 \, \bigl( \Gamma(\alpha-2) \bigr)^2}{\bigl( \Gamma(\alpha) \bigr)^2}}_{1/(\alpha-2)^2} \,
\underbrace{\frac{\Gamma(2\alpha)}{(2\alpha-3) \, \Gamma\bigl(2(\alpha-2)\bigr)}}_{= (2\alpha-1)(2\alpha-2)(2\alpha-4)}
= d \,\, \frac{4(2\alpha-1)(\alpha-1)}{\alpha-2}
\]
hence the upper bound $\Tr \cI\bigl( \beta_{\alpha,\alpha}^{(d)} \bigr) \leq 16 d \alpha$ for $\alpha \geq 3$, which concludes the proof.
\end{proof}

	\section{Proof of Theorem~\ref{thm:nlridge} and of Corollary~\ref{cor:RNL-standard}}
	\label{apd:thm1proof}

We start with the proof of Corollary~\ref{cor:RNL-standard}. \medskip

\begin{proof}
We assume that the Gram matrix $G_T$ is full rank; otherwise,
we may adapt the proof below by resorting to Moore-Penrose pseudoinverses, just
as we do in Appendix~\ref{apd:thm2proof} for the proof of Theorem~\ref{thm:nlridgeadapted}.

Theorem~\ref{thm:nlridge} indicates that
\[
\sum_{t=1}^T (\y_t - \uhat_t \cdot \x_t)^2
\leq \inf_{\ud \in \Rd} \left\{ \sum_{t=1}^T (\y_t - \ud \cdot \x_t)^2 +\lambda \norm{\ud}^2 \right\} +  B^2 \sum_{k=1}^d \log \!\left(1 + \frac{\eigen_{k}(\G_T)}{\lambda} \right).
\]
Now, as in~\eqref{eq:Ab1}, we have a closed-form expression of the unique vector achieving the
following, infimum:
\[
\inf_{\ud \in \Rd} \left\{ \sum_{t=1}^T (\y_t - \ud \cdot \x_t)^2 \right\}
= \sum_{t=1}^T (\y_t - \ud^\star \cdot \x_t)^2
\]
Namely, $\ud^\star = \G_T^{-1} \B_T$, so that
\begin{align}
\nonumber
\norm{\ud^\star} = \bnorm{\G_T^{-1/2} \, \G_T^{-1/2} \B_T}
\leq \lambda_1\bigl( \G_T^{-1/2} \bigr) \, \norm{\G_T^{-1/2} \B_T}
& = \frac{1}{\sqrt{\lambda_d(\G_T)}} \, \norm{\G_T^{-1/2} \B_T} \\
\label{eq:ustar}
& \leq \frac{1}{\sqrt{\lambda_d(\G_T)}} \, B\sqrt{T}\,,
\end{align}
where we used, for the final inequality,
an elementary argument of orthogonal projection that is at the heart
of the proof of Theorem~\ref{thm:nlridgeadapted}: see~\eqref{eq:elementaryortho}
and the sentence after it.
In addition, Jensen's inequality (or the alternative treatment of
\citealp[page 320]{Cesa-Bianchi2006}) indicates that
\[
\sum_{k=1}^d \log \!\left(1 + \frac{\eigen_{k}(\G_T)}{\lambda} \right)
\leq d \log \!\left(1 + \frac{\sum_{k=1}^d \eigen_{k}(\G_T)}{d\lambda} \right)
= d \log \!\left(1 + \frac{\Tr(\G_T)}{d\lambda} \right)
\leq d \log \!\left(1 + \frac{T X^2}{d\lambda} \right)
\]
where $\Tr$ is the trace operator.
All in all, we get
\[
\sum_{t=1}^T (\y_t - \uhat_t \cdot \x_t)^2
\leq \inf_{\ud \in \Rd} \left\{ \sum_{t=1}^T (\y_t - \ud \cdot \x_t)^2 + \right\}
+ \lambda \norm{\ud^\star}^2 + d B^2 \log \!\left(1 + \frac{T X^2}{d\lambda} \right)
\]
and the claimed bound follows by substituting the bound~\eqref{eq:ustar}.
\end{proof}

    Now, we move to the proof of Theorem~\ref{thm:nlridge}, which
    we essentially extract from~\citet[Chapter~11]{Cesa-Bianchi2006}. We merely
    provide it because we will later need the first inequality of~\eqref{eqn:firstthm1}
    in the proof of Theorem~\ref{thm:nlridgesequenceadapted} and we wanted this article
    to be self-complete. But of course, this is extremely standard content
    and it should be skipped by any reader familiar with the basic results of sequential
    linear regression. \medskip

	\begin{proof}
		We successively prove the following two inequalities,
		\begin{equation}
		\label{eqn:firstthm1}
		\regret_T(\ud)
		\leq \lambda \norm{\ud}^2 + \sum_{t=1}^T \y_t^2 \, \transpose{\x}_t \A_t^{-1} \x_t
		\leq \lambda \norm{\ud}^2 + B^2 \sum_{k=1}^d \log \!\left(1 + \frac{\eigen_{k}(\G_T)}{\lambda} \right) \\
		\end{equation}

		\noindent{\emph{Proof of the first inequality in}~\eqref{eqn:firstthm1}.}~~We denote by $\pot_{t-1}$
		the cumulative loss up to round $t-1$ included, to which we add the regularization term:
		$$\pot_{t-1}(\ud) = \sum_{s=1}^{t-1} (\y_s - \ud \cdot \x_s)^2 + \lambda \norm{\ud}^2$$
		For all $t \geq 1$, we denote by
		\[
		\uc_t \in \argmin_{\ud \in \Rd} \left\{ \sum_{s=1}^{t-1} (\y_s - \ud \cdot \x_s)^2
		+ \lambda \norm{\ud}^2 \right\} =
		\argmin_{\ud \in \Rd} \pot_{t-1}(\ud)\,,
		\]
		the vector output by the (ordinary) ridge regression; that is,
		when no $(\ud \cdot \x_t)^2$ term is added to the regularization.
		In particular, $\uc_1 = (0,\ldots,0)^{\transp}$.
		By the very definition of $\uc_{T+1}$, for all $\ud \in \Rd$,
		\[
		\pot_{T}(\uc_{T+1}) \leq \sum_{t=1}^{T} (\y_t - \ud \cdot \x_t)^2
		+ \lambda \norm{\ud}^2\,,
		\]
		so that, for all $\ud \in \Rd$,
		\begin{align*}
		\regret_T(\ud)
		& \leq \sum_{t=1}^T (\y_t - \yhat_t)^2 + \lambda \norm{\ud}^2 - \pot_{T}(\uc_{T+1}) \\
		& = \lambda \norm{\ud}^2 + \sum_{t=1}^T \bigl( (\y_t - \yhat_t)^2 + \pot_{t-1}(\uc_t) - \pot_{t}(\uc_{t+1}) \bigr)\,,
		\end{align*}
		where the equality comes from a telescoping argument together with
		$\pot_0(\uc_0) = 0$.
		We will prove by means of direct calculations that
		\begin{equation}
		\label{eq:directcalc-ridgeNL}
		(\y_t - \yhat_t)^2 + \pot_{t-1}(\uc_t) - \pot_{t}(\uc_{t+1}) = \transpose{(\uc_{t+1} - \uhat_t)} \A_t (\uc_{t+1} - \uhat_t) - \transpose{(\uhat_t - \uc_{t})} \A_{t-1} (\uhat_{t} - \uc_t)\,;
		\end{equation}
		the first inequality in~\eqref{eqn:firstthm1} will then be obtained, as the second term in~\eqref{eq:directcalc-ridgeNL} is negative and as the first term in~\eqref{eq:directcalc-ridgeNL} can be rewritten as $\y_t^2 \, \transpose{\x}_t \A_t^{-1} \x_t$ thanks to the equality~\eqref{eqn:inovation-substitution} below,
		which states $\A_t (\uc_{t+1} - \uhat_t) = \y_t \x_t$.

		To prove~\eqref{eq:directcalc-ridgeNL}, we recall the closed-form expression~\eqref{eq:Ab1},
		that is, $\uhat_t = \A_{t}^{-1} \B_{t-1}$, and note that
		we similarly have $\uc_{t+1} = \A_{t}^{-1} \B_{t}$.
		Now, $\pot_t$ rewrites, for all $\ud \in \Rd$,
		\[
		\pot_{t}(\ud) = \left( \sum_{s=1}^{t} \y_s^2 \right) - 2 \transpose{\B}_t \ud + \transpose{\ud} \A_t \ud\,,
		\]
		so that the minimum of this quadratic form, achieved at $\ud = \uc_{t+1} = \A_{t}^{-1} \B_{t}$, equals
		\[
		\pot_{t}(\uc_{t+1}) = \left( \sum_{s=1}^{t} \y_s^2 \right) - 2 \underbrace{\transpose{\B}_t \A_{t}^{-1}}_{= \transpose{\uc_{t+1}}}
		\A_t \uc_{t+1} + \transpose{\uc}_{t+1} \A_{t} \uc_{t+1}
		= \left( \sum_{s=1}^{t} \y_s^2 \right) - \transpose{\uc}_{t+1} \A_{t} \uc_{t+1}\,.
		\]
		In particular,
		\begin{equation}
		\label{eq:minfq}
		\pot_{t-1}(\uc_t) - \pot_{t}(\uc_{t+1})
		= - y_t^2 + \transpose{\uc}_{t+1} \A_{t} \uc_{t+1} - \transpose{\uc}_{t} \A_{t-1} \uc_{t}\,.
		\end{equation}

		We now expand the first term in~\eqref{eq:directcalc-ridgeNL}.
		To that end, we use
		that from the closed-form expressions of $\uhat_t$ and $\uc_{t+1}$,
		\begin{equation}
		\label{eqn:inovation-substitution}
		\A_t (\uc_{t+1} - \uhat_t) = \A_t \bigl( \A_{t}^{-1} \B_{t} - \A_{t}^{-1} \B_{t-1} \bigr) = \B_t - \B_{t-1} =
		\y_t \x_t\,.
		\end{equation}

		Therefore, $\y_t \yhat_t = \y_t \transpose{\x}_t \uhat_t = \transpose{(\uc_{t+1} - \uhat_t)} \A_t \uhat_t$ and
		\begin{align}
		\nonumber
		(\y_t - \yhat_t)^2 = \y^2_t - 2 \y_t \yhat_t + \yhat^2_t
		& = \y^2_t - 2\transpose{(\uc_{t+1} - \uhat_t)} \A_t \uhat_t
		+ \transpose{\uhat}_t \x_t \transpose{\x}_t \uhat_t \\
		\label{eq:sqexpand}
		& = \y^2_t - 2\transpose{(\uc_{t+1} - \uhat_t)} \A_t \uhat_t
		+ \transpose{\uhat}_t (\A_t - \A_{t-1}) \uhat_t\,,
		\end{align}
		where in the last equality we used that by definition $\A_t - \A_{t-1} = \x_t \transpose{\x}_t$.

		Putting~\eqref{eq:minfq} and~\eqref{eq:sqexpand} together, we proved
		\begin{align*}
		\lefteqn{(\y_t - \yhat_t)^2 + \pot_{t-1}(\uc_t) - \pot_{t}(\uc_{t+1})} \\
		& = - 2\transpose{(\uc_{t+1} - \uhat_t)} \A_t \uhat_t
		+ \transpose{\uhat}_t (\A_t - \A_{t-1}) \uhat_t
		+ \transpose{\uc}_{t+1} \A_{t} \uc_{t+1} - \transpose{\uc}_{t} \A_{t-1} \uc_{t} \\
		& = \transpose{\uc}_{t+1}  \A_{t} \uc_{t+1} - 2 \transpose{\uc_{t+1} } \A_t \uhat_t + \transpose{\uhat}_t \A_t \uhat_t - \bigl( \transpose{\uhat}_t \A_{t-1} \uhat_t - 2 \transpose{\uhat_{t}} \underbrace{\A_{t} \uhat_t}_{=\A_{t-1} \uc_t} + \transpose{\uc}_{t}  \A_{t-1} \uc_{t} \bigr)\,.
		\end{align*}
		In the last equation, we are about to use the equality $\A_t \uhat_t = \A_{t-1} \uc_t = \B_{t-1}$, which we get from the closed-form expressions of $\uhat_t$ and $\uc_{t}$. We then recognize the desired difference between two quadratic forms:
		\begin{align*}
		\lefteqn{(\y_t - \yhat_t)^2 + \pot_{t-1}(\uc_t) - \pot_{t}(\uc_{t+1})} \\
		& = \bigl( \transpose{\uc}_{t+1}  \A_{t} \uc_{t+1} - 2 \transpose{\uc_{t+1} } \A_t \uhat_t + \transpose{\uhat}_t \A_t \uhat_t \bigr)- \bigl( \transpose{\uhat}_t \A_{t-1} \uhat_t - 2 \transpose{\uhat_{t}} \A_{t-1} \uc_t + \transpose{\uc}_{t}  \A_{t-1} \uc_{t} \bigr) \\
		& = \transpose{(\uc_{t+1} - \uhat_t)} \A_t (\uc_{t+1} - \uhat_t) - \transpose{(\uhat_t - \uc_{t})} \A_{t-1} (\uhat_{t} - \uc_t)\,.
		\end{align*}

		\noindent{\emph{Proof of the second inequality in}~\eqref{eqn:firstthm1}.}
		Because $y^2_t \leq B^2$, we only need to prove
		\[
		\sum_{t=1}^T \transpose{\x}_t \A_t^{-1} \x_t
		\leq \sum_{k=1}^d \log \!\left(1 + \frac{\eigen_{k}(\G_T)}{\lambda} \right)\,.
		\]
		Now, Lemma~\ref{lem:linalg} below shows that
		\[
		\sum_{t=1}^T \transpose{\x}_t \A_t^{-1} \x_t
		= \sum_{t=1}^T \left( 1 - \frac{\det(\A_{t-1})}{\det(\A_{t})} \right).
		\]
		We then use $1-u \leq - \log u$ for $u > 0$ and identify a telescoping sum,
		\[
		\sum_{t=1}^T \left( 1 - \frac{\det(\A_{t-1})}{\det(\A_{t})} \right)
		\leq \sum_{t=1}^T \log \frac{\det(\A_{t})}{\det(\A_{t-1})}
		= \log \frac{\det(\A_{T})}{\det(\A_{0})}\,.
		\]
		All in all, we proved so far
		\[
		\sum_{t=1}^T \transpose{\x}_t \A_t^{-1} \x_t
		\leq \log \frac{\det(\A_{T})}{\det(\A_{0})}\,,
		\]
		and may conclude by noting that
		\[
		\det(\A_{T}) = \det(\lambda \, \Id + \G_T) = \prod_{k=1}^{d} (\lambda + \eigen_{k}(\G_T))
		\qquad \mbox{and} \qquad
		\det(\A_{0}) = \det(\lambda \, \Id) = \lambda^d\,.
		\]
		\vspace{-1cm}

	\end{proof}

	\begin{lemma}
		\label{lem:linalg}
		Let $V$ an arbitrary $d \times d$ full-rank matrix, let $\ud$ and $\vd$ two arbitrary vectors of $\Rd$, and let $\U = \V - \ud \transpose{\vd}$. Then
		\[
		\transpose{\vd} \V^{-1} \ud = 1 - \frac{\det(\U)}{\det(\V)} \,.
		\]
	\end{lemma}

	\begin{proof}
		If $V = \Id$, we are left to show that $\det(\Id - \ud \transpose{\vd}) = 1 - \transpose{\vd} \ud$. The result follows from taking the determinant of every term of the equality
		\[
		\left[ \begin{array}{cc}
		\Id & 0 \\
		\transpose{\vd} & 1
		\end{array}
		\right]
		\left[ \begin{array}{cc}
		\Id-\ud\transpose{\vd} & -\ud \\
		0 & 1
		\end{array}
		\right]
		\left[ \begin{array}{cc}
		\Id & 0 \\
		-\transpose{\vd} & 1
		\end{array}
		\right]
		=
		\left[ \begin{array}{cc}
		\Id & -\ud \\
		0 & 1-\transpose{\vd}\ud
		\end{array}
		\right] \,.
		\]
		Now, we can reduce the case of a general $\V$ to this simpler case by noting that
		\[
		\det(\U) = \det\bigl(\V-\ud\transpose{\vd}\bigr)
		= \det(\V) \det\Bigl(\Id-\bigl(\V^{-1}\ud\bigr)\transpose{\vd}\Bigr)
		= \det(\V) \bigl(1-\transpose{\vd}\V^{-1}\ud\bigr)\,.
		\]
		\vspace{-1cm}

	\end{proof}

\section{Technical complements to Section~\ref{sec:beforehandfeatures}}
\label{sec:techbeforehand}

In this section we provide some additional discussions to those of
Remark~\ref{rk:Bartlett} (Section~\ref{sec:Bartlett})
and also extend the proof of Theorem~\ref{thm:nlridgeadapted}
to work in the general case (Section~\ref{apd:thm2proof}).

\subsection{Complements to Remark~\ref{rk:Bartlett}}
\label{sec:Bartlett}

We detail here why the derivation of a closed-form bound as
led by \citet{Bartlett2015} only entails a bound of
the order of $2 d B^2 \ln T$ and why it cannot easily be improved.

Indeed, Theorem~5 by \citet{Bartlett2015} indicates, in the case where $d=1$ and $B=1$, that
\begin{equation}
\label{eq:Bartlett}
\forall \, T \geq 1, \qquad \regret^\star_T \leq f(T)
\end{equation}
for any function $f: \{1,2,\ldots\} \to \R_+$ satisfying $\smash{e^{-f(T)/2} \leq f(T+1)-f(T)}$ for all $T\geq 1$. As they showed, the function $f(T) = 2\log(1+T/2)+1$ is a suitable choice, but it leads to
the extra multiplicative factor of~$2$ that we pointed out.

However, this choice for $f$ does not seem to be easily improvable;
for instance, functions $f$ of the form $T \mapsto a \log T + b$ for some $a<2$ and $b\in \R$ are such
that
\[
e^{-f(T)/2} = \Theta_T\bigl(T^{-a/2}\bigr)
\qquad \mbox{and} \qquad
f(T+1) - f(T) = a\log\!\bigg(1+\frac{1}{T}\bigg) = \mathcal{O}_T\bigl( T^{-1} \bigr)\,,
\]
hence, are not suitable choices for the bound~\eqref{eq:Bartlett}.

	\subsection{Proof of Theorem~\ref{thm:nlridgeadapted} in the general case}
	\label{apd:thm2proof}

In this section we extend the proof of Theorem~\ref{thm:nlridgeadapted},
provided only in the case of a full-rank Gram matrix $G_T$ in
Section~\ref{sec:beforehandfeatures}, to the general case of
a possibly non-invertible Gram matrix $G_T$.

To that end, we first explain how the closed-form expression~\eqref{eqn:nlridgeadaptedcf} is derived.
	We rewrite the definition equation~\eqref{eq:nlridgeadapted} of $\uhat_t$ as
	\[
	\uhat_t \in \argmin_{\ud \in \Rd} \bigl\{ \transpose{\ud}(\lambda \, \G_T + \G_t) \ud -2 \transpose{\B}_{t-1} \ud \bigr\}\,.
	\]
Because the matrix $\lambda \, \G_T + \G_t$ is positive semidefinite, the considered $\argmin$ is also the set of values $\ud'$ where the gradient vanishes: $(\lambda \, \G_T + \G_t) \ud' = \B_{t-1}$. This system is possibly under-defined because $\ud' \in \Rd$ and $\lambda \, \G_T + \G_t$ is a matrix of size $\d \times \d$, possibly not full rank. The system has at least one solution but the one with minimal Euclidean norm is given
by the Moore-Penrose inverse, see Corollary~\ref{cor:pseudoinverse}~\ref{eqn:minnormMPI}:
\[
\uhat_t = \bigl(\lambda\, \G_T + \G_t \bigr)^\dagger \B_{t-1}\,.
\]
We may now turn to the general proof of Theorem~\ref{thm:nlridgeadapted}.
For an integer $k \geq 1$, we denote therein by $\I_k$ the $k \times k$ identity matrix. \\

\begin{proof}
As a consequence of the spectral theorem applied to the symmetric matrix $\G_T$,
there exists a matrix $\U$ of size $d \times r_T$ and a full rank square matrix $\Sig$ of size $r_T \times r_T$ such that
$\transpose{\U} \U = \I_{\rT}$ and $\G_T = \U \Sig \transpose{\U}$.
We could even impose that the matrix $\Sig$ be diagonal but this property will not be used in this proof.

We will apply the (already proven) bound of Theorem~\ref{thm:nlridgeadapted} in the full rank case.
To that end, we consider the modified sequence of features
	\[
	\xt_t = \transpose{\U} \x_t
	\]
	and first prove that the strategy~\eqref{eq:nlridgeadapted} on the $\xt_t$ leads to
	the same forecasts as the same strategy on the original features $\x_t$; that is,
	\[
	\ut_t \cdot \xt_t = \uhat_t \cdot \x_t\,, \qquad \mbox{where} \qquad
	\ut_t \in \argmin_{\vd \in \RrT} \left\{\sum_{s=1}^{t-1} (\y_s - \vd \cdot \xt_s)^2 + (\vd \cdot \xt_t)^2+ \lambda \sum_{s=1}^T (\vd \cdot \xt_s)^2 \right\}.
	\]
It suffices to prove $\U \ut_t = \uhat_t$, which we do below. Then, from this equality and
the definition $\xt_t = \transpose{\U} \x_t$, we have, as desired,
\[
\ut_t \cdot \xt_t = \ut_t \cdot \bigl( \transpose{\U} \x_t \bigr) = \bigl( \U \ut_t \bigr) \cdot \x_t = \uhat_t \cdot \x_t\,.
\]

Now, to prove $\U \ut_t = \uhat_t$, we resort to the closed-form expression~\eqref{eqn:nlridgeadaptedcf}, which gives that
	\[
	\U \ut_t = \U \left( \lambda \, \sum_{s=1}^{T} \xt_s \transpose{\xt_s} + \sum_{s=1}^{t} \xt_s \transpose{\xt_s} \right)^{\!\! \dagger} \sum_{s=1}^{t-1} y_s \xt_s
	= \U
	\Bigl( \transpose{\U}(\lambda \G_T +  \G_{t}) \U \Bigr)^{\! \dagger}
	\transpose{\U}
	\B_{t-1}\,.
	\]
To simplify this expression, we use twice the property of Moore-Penrose pseudoinverses stated in Corollary~\ref{cor:pseudoinverse}~\ref{eqn:mpproduct},
once with $\M = {\U}$ and the second time with $\N = \transpose{\U}$,
which both satisfy the required condition for Corollary~\ref{cor:pseudoinverse}~\ref{eqn:mpproduct}, as well as the matrix equalities in Corollary~\ref{cor:pseudoinverse}~\ref{eq:MPI-transp}, and we get
\[
\U \Bigl( \transpose{\U}(\lambda \G_T +  \G_{t}) \U \Bigr)^{\! \dagger} \transpose{\U}
=
\Bigl(\U \transpose{\U}(\lambda \G_T +  \G_{t}) \U \transpose{\U} \Bigr)^{\! \dagger}
=
(\lambda \G_T +  \G_{t})^\dagger\,,
\]
where the last equality comes from
\begin{equation}
\label{eq:UU-G-UU}
\U \transpose{\U}(\lambda \G_T +  \G_{t}) \U \transpose{\U} = \lambda \G_T +  \G_{t}\,.
\end{equation}
Indeed, from $\transpose{\U} \U = \I_{\rT}$ we get $\U \transpose{\U} = P_{\image{\G_T}}$, the orthogonal projector on the image of $\G_T$; we recall in~\eqref{eqn:covariancerangesseq} why $\image{\G_t} \subseteq \image{\G_T}$, which implies
$\U \transpose{\U} (\lambda \G_T +  \G_{t}) = \lambda \G_T +  \G_{t}$. Transposing this leads to
$(\lambda \G_T +  \G_{t}) \U \transpose{\U} = \lambda \G_T +  \G_{t}$, from which
the desired equality~\eqref{eq:UU-G-UU} follows by a left multiplication
again by $\U \transpose{\U} = P_{\image{\G_T}}$.

	We may now apply the bound of the Theorem~\ref{thm:nlridgeadapted} in the full rank case on feature sequences $\xt_1,\ldots,\xt_T \in \RrT$ and observations $\y_1,\ldots,y_T \in [-B,B]$; this is because the associated Gram matrix $\transpose{\U} \G_T \U = \Sig$ is now full rank. We get, for all $\vd \in \RrT$,
	\begin{equation}
	\label{eqn:thm2boundrT}
	\sum_{t=1}^T (\y_t - \uhat_t \cdot \x_t)^2 = \sum_{t=1}^T (\y_t - \ut_t \cdot \xt_t)^2 \leq \sum_{t=1}^T (\y_t - \vd \cdot \xt_t)^2 + \lambda T B^2 + r_T B^2 \log \! \left(1+\frac{1}{\lambda}\right).
	\end{equation}
To conclude the proof, its only remains to show that
\begin{equation}
\label{eqn:infequal}
\inf_{\vd \in \RrT} \sum_{t=1}^T (\y_t - \vd \cdot \xt_t)^2
= \inf_{\ud \in \R^d} \sum_{t=1}^T (\y_t - \ud \cdot \x_t)^2\,.
\end{equation}
Now, a basic argument of linear algebra, recalled in~\eqref{eqn:covariancerange} of
Appendix~\ref{apd:covariancesequence}, indicates
$\image{\G_t} = \image{\X_t}$. Together with the inclusion $\image{\G_t} \subseteq \image{\G_T}$
and the fact that $\U \transpose{\U} = P_{\image{\G_T}}$, both already used above,
we get $\U \transpose{\U} \x_t = \x_t$.
A direct consequence is that for any $\ud$ in $\Rd$,
\[
\ud \cdot \x_t = \ud \cdot \bigl( \U \transpose{\U} \x_t \bigr) =
\bigl( \transpose{\U} \ud \bigr) \cdot \bigl( \transpose{\U} \x_t \bigr) =
\bigl( \transpose{\U} \ud \bigr) \cdot \xt_t\,,
\]
from which~\eqref{eqn:infequal} follows, by considering $\vd = \transpose{\U} \ud$ and by
the surjectivity of $\transpose{\U}$ onto $\RrT$ (recall that $\U$ and
$\transpose{\U}$ are of rank $r_T$).
\end{proof}

\section{Proof of Theorem~\ref{thm:nlridgesequenceadapted}}
\label{apd:nlridgeadaptatedproof}

We recall in Appendix~\ref{apd:covariancesequence}
many basic properties of Gram matrices and Moore-Penrose pseudoinverses
to be used in the proof below. \medskip

	\begin{proof}
		We successively prove the following two inequalities,
		\begin{equation}
		\label{eqn:firstthm3}{}
		\regret_T(\ud)
		\leq \sum_{t=1}^T \y_t^2 \transpose{\x}_t \G_t^\dagger \x_t
		\leq B^2 \sum_{k=1}^{r_T} \log \bigl(\eigen_k(\G_T) \bigr) + B^2 \sum_{t \in [\![1, T]\!] \cap \jump} \log \! \left(\frac{1}{\eigen_{r_t}(\G_{t})} \right) + r_T B^2\,,
\end{equation}
where actually, the first inequality is a classic inequality already proved by \citet[Theorem~3.2]{FoWa03}.
We provide its derivation for the sake of completeness only.
        \medskip

		\noindent{\emph{Proof of the first inequality in}~\eqref{eqn:firstthm3}.}~~We obtain it as a limit case.
To do so, we start by exactly rewriting the first inequality of~\eqref{eqn:firstthm1}, where a $\lambda > 0$ regularization factor was
considered:
		\begin{equation}
		\label{eqn:firstthm3_1}
		 \sum_{t=1}^T \big(\y_t - \transpose{\x}_t (\lambda \Id + \G_t)^{-1} \B_{t-1} \big)^2 - \sum_{t=1}^T (\y_t - \ud \cdot \x_t)^2 \leq \sum_{t=1}^T \y_t^2 \, \transpose{\x}_t (\lambda \Id + \G_t)^{-1} \x_t + \lambda \norm{\ud}^2\,.
		\end{equation}
Since
\[
\G_t = \X_t \transpose{\X}_t
\qquad \text{where} \qquad
\X_t =
		\bigl[
		\begin{array}{ccc}
		\x_1 &
		\cdots &
		\x_t
		\end{array}
		\bigr]
\]
we note that $\transpose{\x}_t (\lambda \Id + \G_t)^{-1}$ is the last line of the matrix $\transpose{\X}_t \big(\lambda \Id + \X_t \transpose{\X}_t\big)^{-1}$, which tends to $\X^{\dagger}$ when $\lambda \rightarrow 0$ as indicated by Corollary~\ref{cor:pseudoinverse}~\ref{eqn:mplimit}.
Now, $\X^{\dagger} = \transpose{\X}_t \big(\X_t \transpose{\X}_t\big)^\dagger = \transpose{\X}_t \G_t^\dagger$
by Corollary~\ref{cor:pseudoinverse}~\ref{eqn:mphermit}, thus
\[
\lim\limits_{\lambda \rightarrow 0}\transpose{\x}_t (\lambda \Id + \G_t)^{-1} = \transpose{\x}_t \G_t^\dagger\,.
\]
Therefore, the desired inequality for the considered forecaster,
\[
\regret_T(\ud) = \sum_{t=1}^T \big(\y_t - \transpose{\x}_t \G_t^\dagger \B_{t-1} \big)^2 - \sum_{t=1}^T (\y_t - \ud \cdot \x_t)^2 \leq \sum_{t=1}^T \y_t^2 \transpose{\x}_t \G_t^\dagger \x_t\,,
\]
is obtained by taking the limit $\lambda \rightarrow 0$
in~\eqref{eqn:firstthm3_1}. \\

		\noindent{\emph{Proof of the second inequality in}~\eqref{eqn:firstthm3}.}
The first part of our derivation is similar to what is performed in \citet[Theorem~4 of Appendix~D]{Luo16},
while the second part slightly improves on their result thanks to a more careful analysis using however the same ingredients.

		Because $y^2_t \leq B^2$, we only need to prove
		\[
		\sum_{t=1}^T \transpose{\x}_t \G_t^\dagger \x_t
		\leq \sum_{k=1}^{r_T} \log \bigl(\eigen_k(\G_T) \bigr) + \sum_{t \in \jump \cap [\![1, T]\!]} \log \! \left( \frac{1}{\eigen_{r_t}(\G_{t})} \right) + r_T\,.
		\]
		Now, Lemma~\ref{lem:mplinalg} below shows that
		\[
		\sum_{t=1}^T \transpose{\x}_t \G_t^\dagger \x_t
		= \sum_{t=1}^T \left( 1 - \prod_{k=1}^{r_t} \frac{ \eigen_k(\G_{t-1})}{\eigen_k(\G_{t})} \right);
		\]
we assumed with no loss of generality that $\x_1$ is not the null vector, hence all $\G_t$ are at
least of rank~$1$. Indeed, when $\x_t$ is the null vector, all linear combinations result in the same
prediction equal to~$0$ and incur the same instantaneous quadratic loss.

Now, given the definition of the set $\jump$, whose cardinality is $r_T$, we have
$\eigen_{r_t}(\G_{t-1}) = 0$ when $t \in \jump$ (and this includes $t=1$, with the convention
that $\G_0$ is the null matrix), while $r_{t-1} = r_t$ if $t \notin \jump$. Therefore,
		\begin{align*}
				\sum_{t=1}^T \transpose{\x}_t \G_t^\dagger \x_t & \leq
				\sum_{t \in \jump \cap [\![1, T]\!]} \left( 1 - \prod_{k=1}^{r_t} \frac{\eigen_k(\G_{t-1})}{\eigen_k(\G_{t})} \right) + \sum_{t \in [\![1, T]\!] \setminus \jump} \left( 1 - \prod_{k=1}^{r_t} \frac{\eigen_k(\G_{t-1})}{\eigen_k(\G_{t})} \right) \\
				& = r_T + \sum_{t \in [\![1, T]\!] \setminus \jump} \left( 1 - \frac{D_{t-1}}{D_t} \right),
		\end{align*}
where $D_t = \displaystyle{\prod_{k=1}^{r_t} \eigen_k(\G_{t})}$ is the product of the positive eigenvalues of $\G_t$.

Now (this is where our analysis is more careful), using $1-u \leq - \log u$ for $u > 0$, we get an almost telescoping sum,
		\[
		 \sum_{t \in [\![1, T]\!] \setminus \jump} \left( 1 - \frac{D_{t-1}}{D_t} \right)
		\leq \sum_{t \in [\![1, T]\!] \setminus \jump} \ln \frac{D_t}{D_{t-1}}
= \ln \frac{D_T}{D_1} + \sum_{t \in \jump \cap [\![2, T]\!]} \ln \frac{D_{t-1}}{D_{t} }
\]
(note that we dealt separately with $t=1$, which belongs to $\jump$).
Because eigenvalues cannot decrease with $t$, see~\eqref{eqn:covarianceeigevalues},
we have in particular $\eigen_k(\G_{t-1}) \leq \eigen_k(\G_{t})$ for all $1 \leq k \leq r_t -1$.
Thus, for $t \in \jump$ with $t \ne 1$, we have
\[
\ln \frac{D_{t-1}}{D_{t}} \leq \log \! \left( \frac{1}{\eigen_{r_t}(\G_{t})} \right)\,,
\]
Substituting the definition of $D_T$ and the equality $D_1 = \eigen_{r_1}(\G_{1})$, and collecting all bounds
together leads to the second inequality in~\eqref{eqn:firstthm3}.
\end{proof}

The lemma below was essentially stated and proved by \citet[Lemma~D.1]{CBCG05}.

	\begin{lemma}[Rewriting of $\transpose{\x} \A^\dagger \x$]
		\label{lem:mplinalg}
Let $\mathbf{B}$ be a $d \times d$ symmetric positive semidefinite matrix (possibly the null matrix),
let $\x \in \R^d$, and
and let $\A = \mathbf{B} + \x \transpose{\x}$. Denote by $r$ the rank of $\A$ and assume that $r \geq 1$.
Then
		\begin{equation}
		\transpose{\x} \A^\dagger \x = 1 - \prod_{k=1}^{r} \frac{ \eigen_k(\mathbf{B})}{\eigen_k(\mathbf{A})}\,.
		\end{equation}
	\end{lemma}

	\begin{proof}
This lemma is a consequence of the less general Lemma~\ref{lem:linalg}.
As a consequence of the spectral theorem applied to the symmetric matrix $\A$,
there exists a matrix $\U$ of size $d \times r$ and a full rank square matrix $\Sig$ of size $r \times r$ such that
$\transpose{\U} \U = \I_{r}$ and $\A = \U \Sig \transpose{\U}$.
We can and will even impose that the matrix $\Sig$ is diagonal, with diagonal values
equal to $\lambda_1(\A),\ldots,\lambda_r(\A)$, the positive eigenvalues of $\A$.
Let $\Gam = \Sig - \transpose{\U}\x \transpose{(\transpose{\U} \x)}$.
Lemma~\ref{lem:linalg} with $\Gam$, $\Sig$ and $\transpose{\U} \x$ indicates that
\[
\transpose{\x} \bigl( \U \Sig^{-1} \transpose{\U} \bigr) \x
= \transpose{\big(\transpose{\U}\x\big)} \Sig^{-1} \big(\transpose{\U}\x\big) = 1 - \frac{\det(\Gam)}{\det(\Sig)}
\qquad \mbox{where} \qquad
\det(\Sig) = \prod_{k=1}^{r} \eigen_k(\mathbf{A})\,.
\]
Now, it can be easily checked (by noting that all four properties in Proposition~\ref{prop:pseudoinverse} are satisfied)
that $\A^\dagger = \U \Sig^{-1} \transpose{\U}$, so that from the above equality, it suffices to show that
\[
\det(\Gam) = \prod_{k=1}^{r} \eigen_k(\mathbf{B})
\]
to conclude the proof. To do so, we first remark that $\mathbf{B} = \A - \x\transpose{\x} = \U\Sig\transpose{\U} - \x\transpose{\x}$, which yields
\[
	\transpose{\U} \mathbf{B} \U = \U^\top \U\Sig\transpose{\U} \U - \transpose{\U} \x\transpose{\x} \U = \Sig - \transpose{\U} \x\transpose{\x} \U = \Gam\,.
\]
Using again that
$\transpose{\U} \U = \I_r$, we note that $\transpose{\ud} \ud = \transpose{(\U\ud)} \U \ud$
for all $\ud \in \R^r$.
From this and $\transpose{\U} \mathbf{B} \U = \Gamma$, we get in particular
\begin{equation}
	\label{eq:eigenUSig}
\sup_{0 \ne \ud \in \R^r}
\frac{\transpose{\ud} \Gam \ud}{\transpose{\ud} \ud} =
\sup_{0 \ne \ud \in \R^r}
\frac{ \transpose{(\U\ud)}  \mathbf{B} (\U\ud)}{\transpose{(\U\ud)} \U \ud} \,.
\end{equation}
Next we show that
\begin{equation}
	\label{eq:eigenUSigbis}
\sup_{0 \ne \ud \in \R^r}
\frac{ \transpose{(\U\ud)}  \mathbf{B} (\U\ud)}{\transpose{(\U\ud)} \U \ud}
= \sup_{0 \ne \vd \in \R^d}
\frac{ \transpose{\vd}  \mathbf{B} (\vd)}{\transpose{\vd} \vd}\,,
\end{equation}
which indicates, together with~\eqref{eq:eigenUSig} and
the characterization~\eqref{eq:eigenvaluecharacterization} of the eigenvalues of symmetric positive semidefinite matrices, that $\mathbf{B}$ and $\Gam$ have the same top $r$ eigenvalues, as claimed.
Now, to show~\eqref{eq:eigenUSigbis},
we recall that for a symmetric matrix $\mathbf{B}$, we have
$\R^d = \ker(\mathbf{B}) \oplus \image{\mathbf{B}}$, so that,
\[
\sup_{0 \ne \vd \in \R^d}
\frac{ \transpose{\vd}  \mathbf{B} (\vd)}{\transpose{\vd} \vd}
= \sup_{0 \ne \vd \in \image{\mathbf{B}}}
\frac{ \transpose{\vd}  \mathbf{B} (\vd)}{\transpose{\vd} \vd}\,.
\]
This leads to~\eqref{eq:eigenUSigbis} via the inclusions
\[
\image{\mathbf{B}} \subseteq \image{\U} \subseteq \R^d
\]
which themselves follow from the inclusions
\[
\image{\mathbf{B}} \subseteq \image{\A} \subseteq \image{\U}\,.
\]
Indeed, $\image{\A} \subseteq \image{\U}$ because $\A = \U \Sig \transpose{\U}$
and $\image{\mathbf{B}} \subseteq \image{\A}$, or equivalently, given
that we are considering symmetric matrices, $\ker{\A} \subseteq \ker{\mathbf{B}}$,
as for all $\by \in \R^d$,
\[
\A \by = 0 \implies \transpose{\by} \A \by = 0
\implies \Bigl[ \transpose{\by} \mathbf{B} \by = 0
\ \mbox{and} \ \transpose{\by}\x \transpose{\x}\by = 0 \Bigr]
\implies \sqrt{\mathbf{B}} \by = 0
\implies \mathbf{B}\by = 0\,,
\]
where we used $\A = \mathbf{B} + \x \transpose{\x}$
to get the second implication,
and where we multiplied $\sqrt{\mathbf{B}} \by$ by $\sqrt{\mathbf{B}}$
to get the final implication.
\end{proof}

\section{Some basic facts of linear algebra}
\label{apd:covariancesequence}

We gather in this appendix some useful results of linear algebra,
that are either reminder of well-known facts or are easy to prove
(yet, we prefer prove them here rather for the proofs above to be
more focused).

\subsection{Gram matrices versus matrices of features}
\label{sec:GramMatr}

Recall that we denoted by
\[
\X_t = \bigl[ \begin{array}{ccc}
\x_1 &
\cdots &
\x_t
\end{array}
\bigr]
\]
the $d \times t$ matrix consisting the first $t$ features.
By definition,
\[
\image{\X_t} = \spn\{\x_1, \ldots, \x_t\}
\qquad \mbox{and} \qquad
\G_t = \X_t \transpose{\X}_t\,.
\]
The aim of this section is to show that, for all $t \geq 1$,
	\begin{equation}
	\label{eqn:covariancerange}
	\image{\G_t} = \image{\X_t}\,.
	\end{equation}
which in turn implies that for all $t \geq 2$,
	\begin{equation}
	\label{eqn:covariancerangesseq}
	\image{\G_{t-1}} \subseteq \image{\G_t}\,,
	\end{equation}
and that $\rank(\G_{t-1})$ and $\rank(\G_t)$ differ from at most~$1$.

First, as for any (not necessarily square) matrix $\M$ we have
$\image{\M} = \smash{\ker\bigl(\transpose{\M}\bigr)^\perp}$,
we note that~\eqref{eqn:covariancerange}
is equivalent to
$\smash{\ker\bigl(\G_t\bigr)^\perp = \ker\bigl(\transpose{\X}_t\bigr)^\perp}$,
thus to
$\ker\bigl(\G_t\bigr) = \ker\bigl(\transpose{\X}_t\bigr)$.
It is clear by definition of $\G_t$ that
$\ker\bigl(\transpose{\X}_t\bigr) \subseteq \ker\bigl(\G_t\bigr)$;
furthermore, for any vector $\ud \in \Rd$, we have the equality $\transpose{\ud} \G_t \ud = \norm{\transpose{\X}_t \ud}^2$, which yields
the opposite inclusion $\ker\bigl(\G_t\bigr) \subseteq \ker\bigl(\transpose{\X}_t\bigr)$.

The inclusion~\eqref{eqn:covariancerangesseq} follows from~\eqref{eqn:covariancerange}
as by definition, the image of $\X_t$ is generated by the image of $\X_{t-1}$ and $\x_t$.

	\subsection{Dynamic of the eigenvalues of Gram matrices}
\label{sec:dyna}

	The above result gives us an idea of how eigenspaces and eigenvalues of the covariance matrix evolve. Another relationship is the following one: for $t \geq 1$,
	\begin{equation}
	\label{eqn:covarianceeigevalues}
	\eigen_k \bigl(\G_{t-1}\bigr) \leq \eigen_k \bigl(\G_{t}\bigr)\,,
	\end{equation}
	where we recall that $\eigen_k \bigl(\G_{t}\bigr)$ denotes the k\ts{th} eigenvalue of $\G_t$ in decreasing order.
To prove this we remark that for all $\ud \in \Rd$, we have
\[
\transpose{\ud} \G_{t-1} \ud \leq \transpose{\ud} \x_{t} \ud + \transpose{\ud} \G_{t-1} \ud = \transpose{\ud} \G_{t} \ud
\]
and use the fact that for all symmetric positive semidefinite matrices $\M$,
\begin{equation}
\label{eq:eigenvaluecharacterization}
\lambda_k(\M) = \max \Biggl\{\min _{\ud}\left\{\frac{\transpose{\ud} \M \ud}{\transpose{\ud} \ud} \mid \ud\in U{\text{ and }} \ud\neq 0\right\}\Bigg| \,\, U \ \mbox{vector space with} \ \dim(U)=k\Biggr\}
\end{equation}

\subsection{Moore-Penrose pseudoinverses: definition and basic properties}
\label{sec:MPI}

In this appendix, we recall the definition and some basic properties of the Moore-Penrose pseudoinverse.
It was introduced by E.H.\ Moore in 1920 and is a generalization of the inverse operator for non-invertible (and non-square) matrices.

\begin{definition}[Moore-Penrose pseudoinverse]
The Moore-Penrose pseudoinverse of an $m \times n$ matrix $\M$ is a $n \times m$ matrix denoted by $\M^\dagger$ and defined as
\begin{equation*}
    \M^\dagger \eqdef \lim_{\alpha \to 0} \big(\transpose{\M} \M + \alpha I_n)^{-1} \transpose{\M}\,,
\end{equation*}
where $\I_n \in \R^{n \times n}$ is the identity matrix and $\alpha \to 0$ while $\alpha >0$.
\end{definition}

We have the following characterization of $\M^\dagger$.

\begin{proposition} \label{prop:pseudoinverse}
Let $\M$ be a $m \times n$ matrix. Its Moore-Penrose pseudoinverse $\M^\dagger$ is unique and
is characterized as the only $n \times m$ matrix simultaneously satisfying the following
four properties:
\vspace*{-7pt}
    \begin{itemize}[noitemsep,nolistsep]
    \begin{multicols}{2}
        \item[(P1)] \qquad $\M \M^\dagger \M = \M$
        \item[(P2)] \qquad $\M^\dagger \M \M^\dagger = \M^\dagger$
        \item[(P3)] \qquad $\transpose{\bigl(\M \M^\dagger\bigr)} = \M \M^\dagger$
        \item[(P4)] \qquad $\transpose{\bigl(\M^\dagger \M\bigr)} = \M^\dagger \M$
    \end{multicols}
    \end{itemize}
\end{proposition}

The proof can be found in~\citet{penrose1955generalized}. In particular, in our analysis we use the following consequences of Proposition~\ref{prop:pseudoinverse}. (We leave the standard proofs to the reader.)

\begin{corollary} \label{cor:pseudoinverse}
Let $\M$ be a $m \times n$ matrix and $\N$ a $n \times p$ matrix. Then,
\begin{enumerate}[label={(\alph*)},noitemsep,nolistsep,topsep=0pt]
    \item \label{eqn:mphermit}
    $\M^{\dagger} = \transpose{\M} \bigl(\M \transpose{\M} \bigr)^{\! \dagger};$
    \item
    \label{eqn:mpproduct}
    if $\transpose{\M} \M = \I_n$ or $\N \transpose{\N}= \I_n$ then
    $
    \bigl(\M \N \bigr)^{\! \dagger} = \N^\dagger \M^\dagger;
    $
    \item \label{eq:MPI-transp}
    if $\transpose{\M}\M = \I_n$, then
    $
    \transpose{\M} = \M^\dagger$ and
    $\M = \big(\transpose{\M}\big)^\dagger;
    $
    \item \label{eqn:mplimit}
    $\M^{\dagger}
= \displaystyle{\lim_{\alpha \rightarrow 0} } \transpose{\M} \big(\lambda \I_m + \M \transpose{\M}\big)^{-1};$
    \item \label{eqn:minnormMPI}
if the equation $\M \x = \bz$ with unknown $\bz \in \R^m$ admits a solution $\x \in \R^n$,
then $\M^\dagger \bz$ is the solution in $\R^n$ with minimal Euclidean norm.
\end{enumerate}
\end{corollary}

\bibliography{rd}

\end{document}